%% file: main.tex
\documentclass[11pt]{article}

\usepackage[CJKbookmarks=true,
            bookmarksnumbered=true,
  bookmarksopen=true,
  colorlinks=true,
  citecolor=blue,
  linkcolor=blue,
  anchorcolor=blue,
  urlcolor=blue]{hyperref}

\usepackage[letterpaper, margin=1in]{geometry} %on letter-size (8 1/2 x 11 inch) paper %1-inch margins

\usepackage{algorithmic}
\usepackage{array}
\usepackage{lipsum}
\usepackage{multirow}
\usepackage{makecell}
\usepackage{xcolor}
\usepackage{enumerate}
\usepackage{tikz}
\usepackage{amsmath}
\usepackage{amssymb}
\usepackage{ulem}
\usepackage{etoolbox}
\usepackage{authblk}
\usepackage{enumitem}
\newcommand*{\circled}[1]{\lower.7ex\hbox{\tikz\draw (0pt, 0pt)%
    circle (.5em) node {\makebox[1em][c]{\small #1}};}}
\robustify{\circled}
\providecommand{\keywords}[1]
{
  {\small	
  \textbf{{Keywords:}} #1
  }
}
\renewcommand*{\Affilfont}{\normalsize\normalfont}
\usepackage{fancyhdr}
\AtBeginDocument{%
    \addtolength{\footskip}{2.0\baselineskip}%
    \fancyfoot[L]{\textit{\textbf{Preprint --- do not distribute.}}}%
}

\input{commands}

\begin{document}
% \acmSubmissionID{123-A56-BU3}
%\title{Bayesian Privacy: Generalizations, Optimizations, and Lower Bounds}
\title{Trading Off Privacy, Utility and Efficiency in Federated Learning}

\author[a]{Xiaojin Zhang}
\author[b]{Yan Kang}
\author[a]{Kai Chen}
\author[b]{Lixin Fan}
\author[a,b]{Qiang Yang$^\dagger$}

\affil[a]{Hong Kong University of Science and Technology, Hong Kong}
\affil[b]{WeBank, China}
\renewcommand*{\Affilfont}{\small\it} % 
\date{} 

\maketitle
\def\thefootnote{$\dagger$}\footnotetext{Corresponding author}\def\thefootnote{\arabic{footnote}}

\begin{abstract}
Federated learning (FL) enables participating parties to collaboratively build a global model with boosted utility without disclosing private data information. Appropriate protection mechanisms have to be adopted to fulfill the opposing requirements in preserving \textit{privacy} and maintaining high model \textit{utility}. In addition, it is a mandate for a federated learning system to achieve high \textit{efficiency} in order to enable large-scale model training and deployment. We propose a unified federated learning framework that reconciles horizontal and vertical federated learning. Based on this framework, we formulate and quantify the trade-offs between privacy leakage, utility loss, and efficiency reduction, which leads us to the No-Free-Lunch (NFL) theorem for the federated learning system. NFL indicates that it is unrealistic to expect an FL algorithm to simultaneously provide excellent privacy, utility, and efficiency in certain scenarios. We then analyze the lower bounds for the privacy leakage, utility loss and efficiency reduction for several widely-adopted protection mechanisms including \textit{Randomization}, \textit{Homomorphic Encryption}, \textit{Secret Sharing} and \textit{Compression}. Our analysis could serve as a guide for selecting protection parameters to meet particular requirements.

\end{abstract}


\maketitle

% \setion{Related Works}

%----------------------------------------------------------------------------------------
%	Main Part
%----------------------------------------------------------------------------------------
\input{01-Introduction}
\input{02-related}

\input{03-The-Model}

\input{06_Trade_off}

\input{09_Meta_Algorithm}
\input{07_Applications}

\input{08-Conclusion}

\textbf{ACKNOWLEDGMENTS}

%\section{Acknowledgments}
We would like to thank Di Chai, Xiaodian Cheng, Tao Fan, Shaofeng Jiang, Yilun Jin, Liu Yang, Ke Yi, Junxue Zhang for helpful discussions. This work was partially supported by the National Key Research and Development Program of China under Grant 2018AAA0101100 and Hong Kong RGC TRS T41-603/20-R.

%----------------------------------------------------------------------------------------
%	Bibliography
%----------------------------------------------------------------------------------------
%\clearpage
% \bibliography{sample/Reference/BayesianPrivacy, sample/Reference/attack, sample/Reference/encryption, sample/Reference/protection, sample/Reference/divergence, sample/Reference/FL,
% sample/Reference/privacy_def_tradeoff}{}
% \bibliography{references}
% \bibliographystyle{ACM-Reference-Format}

\bibliography{references}
\bibliographystyle{plain}

\clearpage

%----------------------------------------------------------------------------------------
%	Appendix
%----------------------------------------------------------------------------------------
\appendix
\input{App_A_1_Quantitative_Relationship_Utility}
\input{App_A_2_Quantitative_Relationship_Privacy}
\input{App_A_2_1_Lower_Bound_for_Privacy_Leakage_Discrete}
\input{App_A_3_Quantitative_Relationship_Efficiency}

% \input{App_Privacy_Leakage_Upper_Bound}

\input{App_B_Trade_offs}
\input{App_C_1_Randomization}
\input{App_C_2_Paillier}
\input{App_C_3_Secret_Sharing}
\input{App_C_4_Compression}
\end{document}

%% file: commands.tex
\usepackage{afterpage}
\usepackage{makecell}
\usepackage{tabulary}
\usepackage{xcolor}
\usepackage{colortbl}
\usepackage{prettyref}
\usepackage{framed}
\usepackage{booktabs}       % professional-quality tables
\newcommand{\pref}[1]{\prettyref{#1}}

\newcommand{\savehyperref}[2]{\texorpdfstring{\hyperref[#1]{#2}}{#2}}
\newrefformat{eq}{\savehyperref{#1}{Eq. \textup{(\ref*{#1})}}}
\newrefformat{eqn}{\savehyperref{#1}{Eq.~\textup{(\ref*{#1})}}}
\newrefformat{lem}{\savehyperref{#1}{Lemma~\ref*{#1}}}
\newrefformat{assump}{\savehyperref{#1}{Assumption~\ref*{#1}}}
\newrefformat{defi}{\savehyperref{#1}{Definition~\ref*{#1}}}\newrefformat{tab}{\savehyperref{#1}{Table~\ref*{#1}}}
\newrefformat{lemma}{\savehyperref{#1}{Lemma~\ref*{#1}}}
\newrefformat{line}{\savehyperref{#1}{Line~\ref*{#1}}}
\newrefformat{thm}{\savehyperref{#1}{Theorem~\ref*{#1}}}
\newrefformat{corr}{\savehyperref{#1}{Corollary~\ref*{#1}}}
\newrefformat{cor}{\savehyperref{#1}{Corollary~\ref*{#1}}}
\newrefformat{sec}{\savehyperref{#1}{Section~\ref*{#1}}}
\newrefformat{app}{\savehyperref{#1}{Appendix~\ref*{#1}}}
\newrefformat{ex}{\savehyperref{#1}{Example~\ref*{#1}}}
\newrefformat{fig}{\savehyperref{#1}{Figure~\ref*{#1}}}
\newrefformat{alg}{\savehyperref{#1}{Algorithm~\ref*{#1}}}
\newrefformat{rem}{\savehyperref{#1}{Remark~\ref*{#1}}}
\newrefformat{conj}{\savehyperref{#1}{Conjecture~\ref*{#1}}}
\newrefformat{prop}{\savehyperref{#1}{Proposition~\ref*{#1}}}
\newrefformat{proto}{\savehyperref{#1}{Protocol~\ref*{#1}}}
\newrefformat{prob}{\savehyperref{#1}{Problem~\ref*{#1}}}
\newrefformat{claim}{\savehyperref{#1}{Claim~\ref*{#1}}}
\newrefformat{que}{\savehyperref{#1}{Question~\ref*{#1}}}
\newrefformat{op}{\savehyperref{#1}{Open Problem~\ref*{#1}}}
\newrefformat{fn}{\savehyperref{#1}{Footnote~\ref*{#1}}}
\newrefformat{property}{\savehyperref{#1}{Property~\ref*{#1}}}

\usepackage{xspace}
\usepackage{amsmath}
\usepackage{amsthm}
\usepackage{bbm}
\usepackage{algorithm}
\usepackage[algo2e, vlined, ruled]{algorithm2e}
\usepackage{algorithmic}
\usepackage{bm}
\usepackage{makecell}
\usepackage{tabulary}
\DeclareMathOperator*{\argmax}{argmax} % no space, limits underneath in displays

\newcommand{\calD}{\mathcal{S}}
\newcommand{\calL}{\mathcal{L}}
\newcommand{\calN}{\mathcal{N}}

\newcommand{\calA}{\mathcal{A}}

\newcommand{\calM}{\mathcal{M}}

\newcommand{\calW}{\mathcal{W}}
\newcommand{\calU}{\mathcal{U}}
\newcommand{\calO}{\mathcal{B}}
\newcommand{\calRO}{\mathcal{O}}
\newcommand{\R}{\mathbb{R}}

\newcommand{\U}{{\mathcal{U}_k}}

\newcommand{\calP}{\mathcal{P}}

\newcommand{\one}{\mathbbm{1}}
\newcommand{\inner}[1]{\langle#1\rangle}

\newtheorem*{theorem*}{Theorem}

\newtheorem{assumption}{Assumption}[section]

\newtheorem*{lemma*}{Lemma}
\newtheorem{definition}{Definition}[section]

\theoremstyle{definition}

\newtheorem{thm}{Theorem}[section]

\newtheorem{lem}[thm]{Lemma}

\newcommand{\fillcell}{\cellcolor{pink!25}}

\newcommand{\ttwo}{\text}

%% file: 01-Introduction.tex
\section{Introduction}
% The rapid expansion of large-scale datasets has sparked demand for distributed learning. With the proposal of regulations such as General Data Protection Regulation (GDPR), the data owned by one company is not allowed to be disclosed to other company. Federated learning (FL)~\cite{mcmahan2016federated,mcmahan2017communication, konevcny2016federated, konevcny2016federated_new} meets this requirement by allowing multiple parties to train a federated model collaboratively without sharing private data. Its goal is to achieve excellent model performance, and has resulted in significant progress in the development of large machine learning systems.
% Federated learning can be categorized as \textit{horizontal federated learning} (HFL) and \textit{vertical federated learning} (VFL) based on the intersection of data from distinct parties~\cite{yang2019federated, DBLP:journals/ftml/KairouzMABBBBCC21,gu2020federated}. Data are partitioned by examples in HFL, and by features and labels in VFL. HFL is popular in healthcare and mobile applications \cite{rieke2020future, Antunes2022Healthcare,Andrew2018keyboard,lim2020federated}, while VFL is widely adopted in finance and advertisement~\cite{Atarashi2021factorization, tan2020rec, kang2021privacy}.

The rapid expansion of large-scale datasets has sparked a demand for distributed learning. With the enforcement of data privacy regulations such as the General Data Protection Regulation (GDPR), the data owned by one company is not allowed to be disclosed to others. Federated learning (FL)~\cite{mcmahan2016federated,mcmahan2017communication, konevcny2016federated, konevcny2016federated_new,liu2022vertical} meets this requirement by allowing multiple parties to train a machine learning model collaboratively without sharing private data. In recent years, FL has achieved significant progress in developing privacy-preserving machine learning systems. It has been extended from the conventional \textit{horizontal federated learning} (HFL) to the \textit{vertical federated learning} (VFL) scenarios~\cite{yang2019federated, DBLP:journals/ftml/KairouzMABBBBCC21,gu2020federated}. HFL typically involves a large amount of parties with different samples but share the same feature space. While VFL typically involves a handful of parties that own distinct features of the same set of sample instances.  HFL is popular in healthcare and mobile applications \cite{rieke2020future, Antunes2022Healthcare,Andrew2018keyboard,lim2020federated}, while VFL is widely adopted in finance and advertisement~\cite{Atarashi2021factorization, tan2020rec, kang2021privacy}. 

Preserving privacy is of immense practical importance when federating across different parties. Although the private data of each client is not shared with other collaborators, the private information might still be reconstructed by semi-honest parties upon observing the shared model information~\cite{geiping2020inverting,zhao2020idlg,yin2021see,oscar2022split,fu2022label}. The fundamental requirement for maintaining privacy is to keep potential \textit{privacy leakage} below an acceptable level. This is accomplished by reducing the dependence between shared model information and private data. To protect private data of the participants, many protection mechanisms have been proposed, such as \textit{Randomization Mechanism} \cite{geyer2017differentially,truex2020ldp,abadi2016deep}, \textit{Secret Sharing} \cite{SecShare-Adi79,SecShare-Blakley79,bonawitz2017practical}, \textit{Homomorphic Encryption (HE)} \cite{gentry2009fully,batchCryp}, and \textit{Compression Mechanism} \cite{nori2021fast}. However, the adoption of these protection mechanisms might result in a certain amount of \textit{utility loss} and \textit{efficiency reduction}, as compared with a federated model trained without any protection \cite{dwork2014algorithmic,kang2021privacy}. As a result, theoretical analysis of the trade-off between privacy leakage, utility loss, and efficiency reduction is critical for guiding FL practitioners to choose better protection parameters or design smarter FL algorithms.

Motivated by this goal, the work~\cite{zhang2022no} proposed a statistical framework to analyze the privacy-utility trade-off in FL on a rigorous theoretical foundation. However, its trade-off analysis did not involve efficiency, which is a crucial factor in designing FL algorithms. Besides, \cite{zhang2022no} did not consider the VFL setting, which has a broad range of applications in finance and advertisement. These gaps inspire us to investigate the following crucial open problem: \textit{is it possible to design a protection mechanism that simultaneously achieves infinitesimal privacy leakage, utility loss, and efficiency reduction?} Our main finding is crystallized in a No-Free-Lunch (NFL) Theorem (\pref{thm: utility-privacy-efficiency Trade-off_JSD_mt}) that provides a negative answer for this question in certain scenarios under a unified FL framework that reconciles both HFL and VFL settings. For the outline of our work, please refer to \pref{fig: our_framework}.

Our contributions are as follows:
\begin{itemize}
%      \item We provide a unified view for HFL and VFL in terms of \textit{global model}, \textit{privacy}, \textit{information exploited by the attacker}, and \textit{utility measurement} (see \pref{tab: HFL_vs_VFL}), which extends the framework proposed by \cite{zhang2022no} and make it applicable to both horizontal and vertical FL scenarios. We focus on semi-honest adversary (without collusion), which can be any participant including the server and the client. 
%   We show that many popular attacking mechanisms in both HFL and VFL \cite{zhu2020deep, geiping2020inverting, zhao2020idlg, yin2021see, li2021label, fu2022label} fit into the Bayesian inference attacking framework.
    \item We propose a unified FL framework (see \pref{sec: applications}) for HFL and VFL. The unified FL framework provides a conceptual view of the relationship between privacy leakage, utility loss, and efficiency reduction through the lens of the protector and the adversary. Our No-Free-Lunch theorem is formulated based on this framework, and thus it applies to both HFL and VFL.
    \item We provide a No-Free-Lunch theorem (\textbf{Theorem \ref{thm: utility-privacy-efficiency Trade-off_JSD_mt}}) for federated learning, which 
    % states that it is unrealistic to expect a FL algorithm to achieve exceptional privacy, utility, and efficiency simultaneously in certain scenarios. The NFL 
    quantifies the trade-off between privacy leakage (Def. \ref{defi: average_privacy_JSD}), utility loss (Def. \ref{defi: utility_loss}) and efficiency reduction (Def. \ref{defi: efficiency_reduction}). This quantification indicates that the weighted summation of the privacy leakage, utility loss, and efficiency reduction is greater than a problem-dependent non-zero constant. It characterizes the amount of utility and efficiency that is inevitable to lose in the case that the privacy leakage budget is unduly low.  
    % and raises the intriguing prospect of the NFL theorem concerning Bayesian inference attack in federated learning.
%   \item We apply the NFL theorem to analyze the trade-off for widely-adopted protection mechanisms including \textit{Randomization}, \textit{Homomorphic Encryption}, \textit{Secret Sharing} and \textit{Compression}. The trade-off is characterized by the dominant protection parameters. For example, the variance of the added noise for the randomization mechanism, and the compression probability for the compression mechanism. This technique facilitates the quantification of various protection mechanisms, and therefore, allows the selection of protection parameters to adapt to specific requirements.
   \item We apply the NFL theorem to analyze trade-offs between privacy leakage, utility loss, and efficiency reduction for widely-adopted protection mechanisms, including \textit{Randomization}, \textit{Homomorphic Encryption}, \textit{Secret Sharing} and \textit{Compression}. 
%   The trade-off is characterized by privacy protection parameters (e.g., the variance of the added noise for the randomization mechanism and the compression probability for the compression mechanism). %\red{This technique} facilitates the quantification of various protection mechanisms and, therefore, allows the selection of protection parameters to adapt to specific requirements \red{(revise the last sentence by XiaoJin)}.
%   \blue{
The trade-off is characterized via the lower bounds derived for privacy leakage, utility loss, and efficiency reduction.
%   }
\end{itemize}

%% file: 02-related.tex
\section{Related Work}

The related work for \textit{attacking mechanisms and protection mechanisms in federated learning} and \textit{privacy-utility-efficiency trade-off} are briefly reviewed in this section.

\subsection{Attacking Mechanisms and Protection Mechanisms in Federated Learning}\label{sec:related:attack}

% \red{(revise this by yankang.) The notion of Federated Learning (FL) was initially proposed by McMahan et al. aiming to build a machine learning model based on datasets that are distributed across multiple clients \cite{mcmahan2016federated,mcmahan2017communication, konevcny2016federated, konevcny2016federated_new}. }
% FL scenarios were further extended to three categories \cite{yang2019federated}: \textit{horizontal federated learning}, \textit{vertical federated learning}, and \textit{federated transfer learning}. Our framework and analysis are applicable to both horizontal federated learning and vertical federated learning.
We focus on \textit{semi-honest} adversaries who faithfully follow the federated learning protocol but may infer the private information of other participants based on the exposed model information. 
% \cite{zhu2019dlg,zhu2020deep,geiping2020inverting,zhao2020idlg,yin2021see} show that adversaries could use gradient information to restore the private data up to pixel-level accuracy, with distinct settings of prior distributions and conditional distributions.

In HFL, \cite{zhu2019dlg,geiping2020inverting,zhao2020idlg,yin2021see} demonstrate that adversaries could exploit gradient information to restore the private image data to pixel-level accuracy, with distinct settings of prior distributions and conditional distributions. A variety of protection mechanisms have been proposed in HFL to prevent private data from being deduced by adversarial participants. The most popular ones are \textit{Homomorphic Encryption (HE)}~\cite{gentry2009fully,batchCryp}, \textit{Randomization Mechanism}~\cite{geyer2017differentially,truex2020ldp,abadi2016deep}, \textit{Secret Sharing}~\cite{SecShare-Adi79,SecShare-Blakley79,bonawitz2017practical} and \textit{Compression Mechanism} \cite{nori2021fast}. Another school of FL~\cite{gupta2018distributed,gu2021federated} tries to protect privacy by splitting a neural network into private and public models and sharing only the public one~\cite{kang2021privacy,gu2021federated}.

VFL has two kinds of privacy leakage: feature leakage and label leakage. It is challenging for the adversary to infer the features of other parties because, in production VFL, participating parties typically have black-box knowledge about each other. The literature has proposed mainly two kinds of label inference attacks in VFL: the gradient-based~\cite{oscar2022split} and the model-based~\cite{fu2022label} attacks. ~\cite{oscar2022split} also demonstrated three noise-based protections that can prevent gradient-based attacks. \cite{yang2020defending} proposed a data encoding protection mechanism called CoAE that can thwart model-based attacks effectively in some scenarios. Crypto-based protections are widely adopted in federated logistic regression and XGBoost. However, they are seldom applied to VNN that involves complex neural networks for their high communication and computational cost. 
% One related work~\cite{kang2021privacy} proposed a HE-protected interactive layer that protects the outputs of parties' local DNN without protecting gradients. Thus, it can not defend against label inference attacks.

\subsection{Privacy-Utility-Efficiency Trade-off} 
In the past decade, there has been wide interest in understanding the privacy-utility trade-off \cite{du2012privacy, makhdoumi2013privacy, sankar2013utility, wang2016relation, wang2017estimation, rassouli2019optimal, zhang2022no}. \cite{du2012privacy,makhdoumi2013privacy} quantified the privacy-utility trade-off using the solution of the optimization problem. \cite{sankar2013utility} provided a privacy-utility trade-off region for the special case with i.i.d. data sources and known distribution. \cite{rassouli2019optimal} illustrated that the optimal privacy-utility trade-off could be solved using a standard linear program and provided a closed-form solution for the special case when the data to be released is a binary variable. \cite{wang2016relation} measured distortion using the expected Hamming distance between the input and output databases and measured privacy leakage using identifiability, differential privacy, and mutual-information privacy separately. \cite{wang2017estimation} provided a trade-off when utility and privacy were evaluated using $\chi^2$-based information measures. \cite{wang2019adaptive} analyzed the trade-off between the speed of error convergence and the wall-clock time for distributed SGD. \cite{chen2020breaking} analyzed the trade-off between communication, privacy, and accuracy for distributed statistical tasks. The accuracy is measured using statistical mean estimation, frequency estimation, and distribution estimation separately, and the privacy is measured using differential privacy. However, none of these works focus on federated learning and measure utility using model performance.

The work~\cite{zhang2022no} proposed a quantitative trade-off between utility and privacy in horizontal federated learning by exploiting some key properties of the privacy leakage and the triangle inequality of the divergence.
% which evaluates the utility loss via the performance reduction of the learning model, and measures the privacy leakage using the Jensen-Shannon divergence. 
% By exploiting some key properties of the privacy leakage and the triangle inequality of the divergence, it provided a quantitative relationship that holds in general.
\cite{lu2020sharing} evaluated the accuracy-privacy-cost trade-off for federated learning empirically. In this work, we offer a general theoretical analysis of the trade-off between privacy, utility, and efficiency that applies to both HFL and VFL.
%\red{will [32] argue this statement?}

%% file: 03-The-Model.tex
\section{A Unified Federated Learning Framework}\label{sec:framework}

% \red{(revise this by yankang)} In this section, we first introduce a unified federated learning framework, which provides unified terminologies and notations for HFL and VFL. Then we show that our proposed framework covers the widely-used protection mechanisms such as \textit{Randomization}, \textit{Homomorphic Encryption}, \textit{Secret Sharing}, and \textit{Compression}, and the semi-honest attacking mechanisms are developed in some similar manners and fit into our proposed framework \cite{zhu2020deep, geiping2020inverting, zhao2020idlg, yin2021see, li2021label, fu2022label}. Finally we provide formal definitions for privacy leakage, utility loss and efficiency reduction under our proposed framework.

In this section, we first introduce general notations used throughout this work. Then, we propose a unified federated learning framework with a conceptual view of the relationship between privacy leakage, utility loss, and efficiency reduction. Next, we provide formal definitions for key components of this framework, including protection and attacking mechanisms, privacy leakage, utility loss, and efficiency reduction. We then formulate the goal of the protector as a constrained optimization problem and put forward a critical question with which we are concerned. Our No-Free-Lunch theorem provides an answer for this problem formulated based on this framework, and thus it applies to both HFL and VFL.
\begin{table*}[!htp]
\footnotesize
  \centering
  \setlength{\belowcaptionskip}{15pt}
  \vspace{-1em}
  \caption{Table of Notation}
  \label{table: notation}
    \begin{tabular}{cc}
    \toprule
    %\hline
    Notation & Meaning\cr
    \midrule\
    %$J$ & exposure function\cr
    %$g$ & The model\cr
    $\epsilon_{p}$ & Privacy leakage (Def. \ref{defi: average_privacy_JSD})\cr
    $\epsilon_u$ & Utility loss (Def. \ref{defi: utility_loss}) \cr
    $\epsilon_e$ & Efficiency reduction (Def. \ref{defi: efficiency_reduction})\cr
    $D$ & Private information, including private data and statistical information\cr
    $W_{\text{fed}}$ & parameter for the federated model\cr
    $W^{\calRO}_k$ & Unprotected model information of client $k$\cr
    $W^{\calD}_k$ & Protected model information of client $k$\cr
 $P^{\calRO}_k$ & Distribution of unprotected model information of client $k$\cr
 $P^{\calD}_k$ & Distribution of protected model information of client $k$\cr
 $\mathcal W_k^{\calD}$ & Support of $P_k^{\calD}$\cr
 $\mathcal W_k^{\calRO}$ & Support of $P_k^{\calRO}$\cr
 $\calW_k$ & Union of the supports of $P_k^{\calD}$ and $P_k^{\calRO}$\cr
 $F^{\calO}_k$ & Adversary's prior belief distribution about the private information of client $k$\cr
 $F^{\calA}_k$ & Adversary's belief distribution about client $k$ after observing the protected private information\cr
 $F^{\calRO}_k$ & Adversary's belief distribution about client $k$ after observing the unprotected private information\cr
 $\text{JS}(\cdot||\cdot)$ & Jensen-Shannon divergence between two distributions\cr
 $\text{TV}(\cdot||\cdot)$ & Total variation distance between two distributions\cr
 %$f^{\calO}_k$ & The probability density function of $F^{\calO}_k$\cr
 %$f^{\calA}_k$ & The probability density function of $F^{\calA}_k$\cr
 %$f^{*}_k$ & The probability density function of $F^{*}_k$\cr
 %$Y = g_w(X)$ & The output of the model with the information $w$\cr
 %$U = g_{w_d}(X)$ & The output of the model with the information $w_d$\cr
 %$\mathcal W^{*}_k$ & The set of all permissible model information, which is the support of $P^{*}_k$\cr
 %$\mathcal W^{\calD}_k$ & The set of all protected model information, which is the support of $P^{d}_k$\cr
 %$\mathcal W_k$ & Union of $\mathcal W^{*}_k$ and $\mathcal W^{\calD}_k$\cr
 %$h: \mathcal{L}\times\mathcal{L}\rightarrow\mathcal R$ & A non-negative loss function\cr
 %$L(Y,\mathcal{D}) = \frac{1}{n} \sum_{i=1}^n h(y_i, l_i)$ & The empirical loss\cr
 %$L(P^{d}_k,\mathcal{D}) = L(P^{d}_k, \mathcal D) = \int_{s\in\mathcal S}p_S(d)\mathbb E_{w\sim P^{d}_k}|g_w(d) - l_s| \textbf{d}\mu(d)$ & The expected empirical loss\cr
% $\calD$ & data-label pairs\cr
     %\hline
    \bottomrule
    \end{tabular}
\end{table*}

% \subsection{Unified Federated Learning Framework}
\subsection{Notations}
% We adhere to the tradition of using uppercase letters to indicate the random variables, such as $D$, and lowercase letters to denote the specific values they take on. We represent $[K]$ as $\{1,2,\cdots, K\}$. The probability density functions are denoted by lowercase letters like $f$ and $p$, while the distributions are denoted by uppercase letters like $F$ and $P$. 

We adhere to the tradition of using uppercase letters to indicate the random variables, such as $D$, and lowercase letters to denote the specific values they take on. We represent $[K]$ as $\{1,2,\cdots, K\}$. We use lowercase letters such as $f$ and $p$ to denote probability density functions and uppercase letters such as $F$ and $P$ to denote distributions.

The probability density function $f$ at value $d$ is represented by $f_{D_k}(d)$, and the subindex represents the random variable. The conditional density function is denoted by the notation $f_{D_k|W_k}(d|w)$. Let $p$ and $q$ represent the probability densities (or probability masses) of $P$ and $Q$ for distributions $P$ and $Q$ over $\mathbb{R}^n$. The Jensen-Shannon divergence, which is a smoothed variation of the Kullback-Leibler divergence, is defined as $\text{JS}(P||Q) = \frac{1}{2}\left[\text{KL}\left(P, M\right) + \text{KL}\left(Q, M\right)\right]$ with $M = (P + Q)/2$. Let $\text{TV}(P||Q)$ denote the total variation distance between $P$ and $Q$, which is defined as $\text{TV}(P||Q) = \sup_{A\subset\mathbb R^n} |P(A) - Q(A)|$. 
The detailed description of notations is illustrated in Table \ref{table: notation}.

\subsection{The Conceptual View of the Unified Federated Learning Framework}

In this section, we propose a unified federated learning framework that unifies HFL and VFL through a conceptual view of the relationship between privacy leakage, utility loss, and efficiency reduction via the lens of the protector and the adversary. We first introduce HFL and VFL and unify their terminologies and notations. Then, we elaborate on this conceptual view. 

% In this section, we introduce the formal definition for HFL and VFL.

HFL has two representative aggregation implementations: FedAvg and FedSGD, which are mathematically equivalent. Our framework applies to both aggregations. For illustrative purposes, we use FedAvg to explain the procedure of secure horizontal federated learning:

\begin{enumerate}[label=\circled{\arabic*}]
\item With the global model state from the server, each client $k$ trains its local model using its private data set $D_k$, and obtains the local model $\theta_k^{\calRO}$.

\item In order to prevent the semi-honest adversaries from inferring other clients' private information $D_k$ according to $\theta_k^{\calRO}$, each client $k$ adopts a protection mechanism $M$ to convert model $\theta_k^{\calRO} $ to protected model $\theta_k^{\calD}$, and sends $\theta_k^{\calD}$ to the server. 

\item The server aggregates $\theta_k^{\calD}, k=1,\cdots,K$ to generate a new global model $\theta_{\text{agg}}^{\calD}$.

\item Each client $k$ downloads the global model $\theta_{\text{agg}}^{\calD}$ and uses it to update its local model.
\end{enumerate}

The processes \textcircled{1}-\textcircled{4} iterate until the utility of the aggregated model $\theta_{\text{agg}}^{\calD}$ does not improve.\\

\begin{figure*}[h!]
    \centering
    \includegraphics[width=0.44\linewidth]{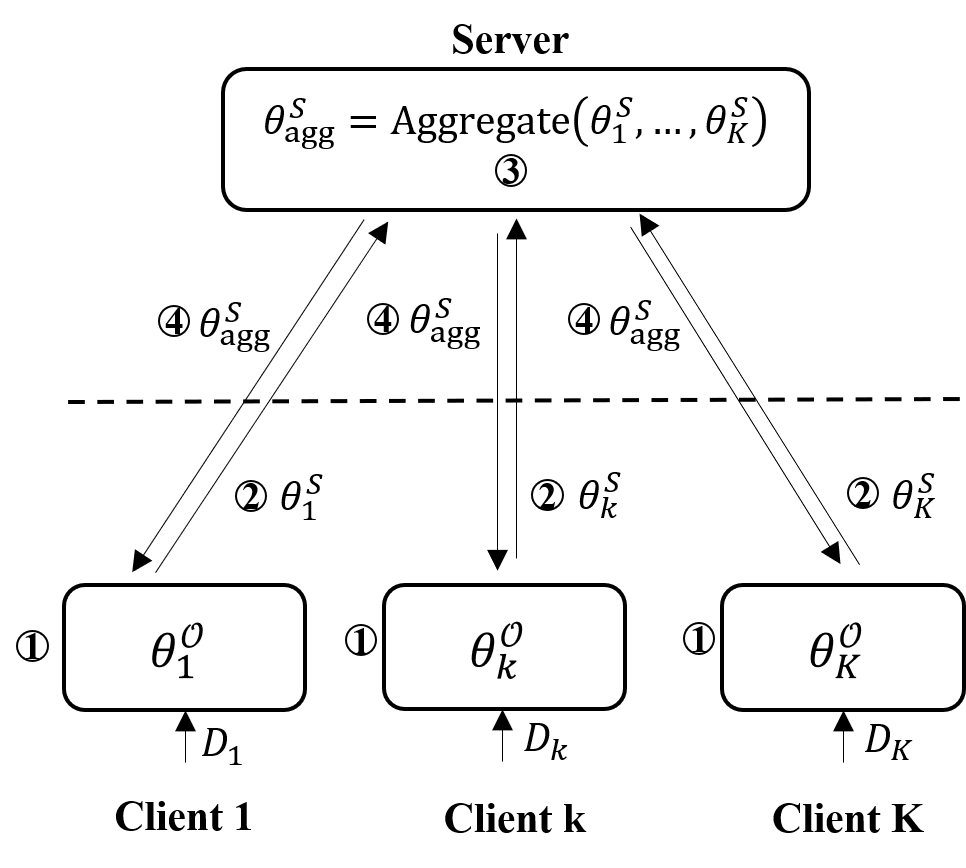}
    % \hspace{2em}
    \includegraphics[width=0.43\linewidth]{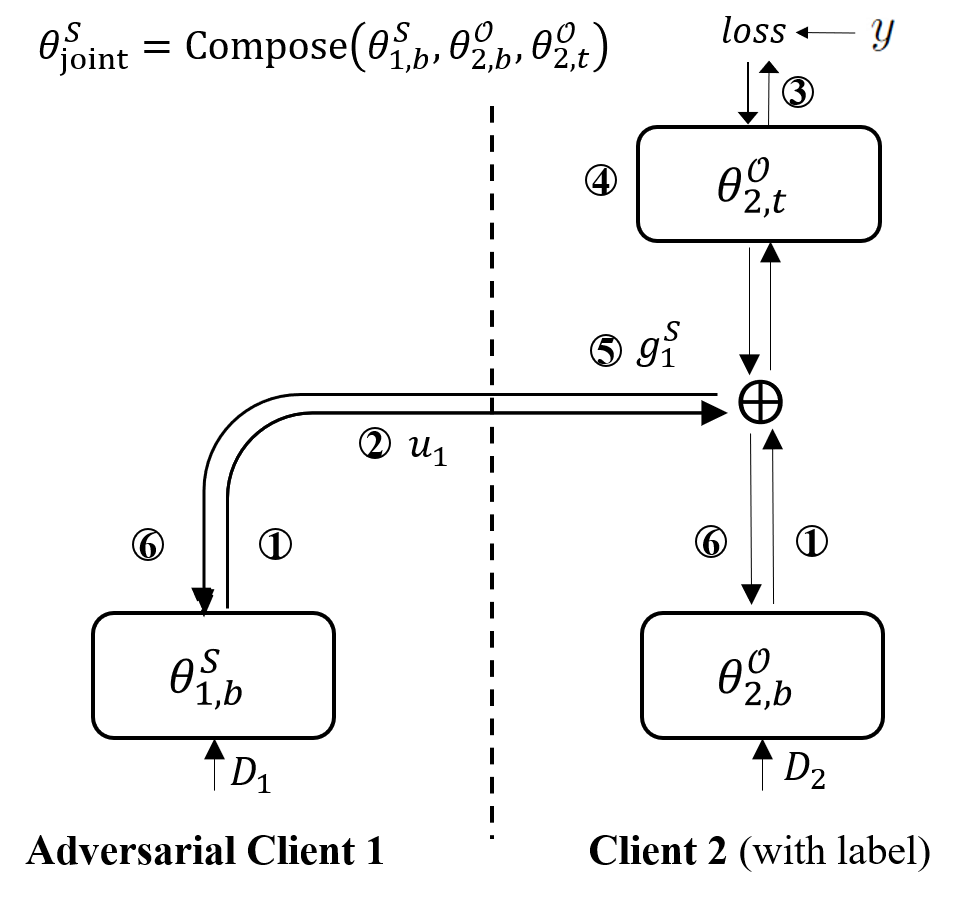}
      \vspace{-1em}
    \caption{Illustration of the HFL procedure (left) and VFL procedure (right). For the HFL setting, we consider the server as the adversary who aims to infer individual clients' private data. For the VFL setting, we consider client 1 (has no labels) as the adversary who wants to recover the labels owned by client 2.}
    \label{fig:fl_proc}
\end{figure*}

% \red{$u_1$ should also be $u_1^d$?}

In VFL, without loss of generality, we consider a 2-client scenario where client 1 (has no labels) is the adversary who wants to recover the labels owned by client 2. The overall secure vertical federated learning procedure is illustrated on the right panel of Fig. \ref{fig:fl_proc} and summarized as follows:
\begin{enumerate}[label=\circled{\arabic*}]
\item Client 1 and 2 compute their model outputs $u_1$ and $u_2$ based on their local bottom model information using their corresponding local data sets, respectively.
\item Client 1 sends $u_1$ to client 2.
\item Client 2 aggregates $u_1$ and $u_2$ and then feeds the aggregated result to its top model $\theta_{2,t}^{\calRO}$ to generate the predicted labels, which further is used to compute loss against the ground truth label $y$.
\item Client 2 updates its top model $\theta_{2,t}^{\calRO}$ and continues to compute the derivatives $g_1$ and $g_2$ of the loss $L$ w.r.t. $u_1$ and $u_2$, respectively.
\item In order to prevent the adversarial client 1 from inferring labels of client 2 based on $g_1$, client 2 adopts protection mechanism $M$ to convert $g_1$ to protected $g_1^{\calD}$, and sends $g_1^{\calD}$ to client 1.
\item Clients 1 and 2 update their local bottom model $\theta_{1,b}^{\calD}$ and $\theta_{2,b}^{\calRO}$ based on $g_1^{\calD}$ and $g_2$, respectively. 
\end{enumerate}
The processes \textcircled{1}-\textcircled{6} iterate until the utility of the joint model $\theta_{\text{joint}}^{\calD}$ does not improve.

For the convenience of our analysis, we unify the terminologies and notations used for HFL and VFL (summarized in \pref{tab: HFL_vs_VFL}). The final trained model in HFL is typically called \textit{aggregated model}, while the one in VFL is called \textit{joint model}. We refer to them as \textit{federated model}. The private information of HFL typically involves private features, labels, and statistical information; the one we study in this work for VFL is private labels. We refer to them all as \textit{private information}. In HFL,
the information exploited by the attacker to infer the private information of protector $k$ is $\theta_k^{\calD}$ if the attacker is the server, and is $\theta_\text{agg}^{\calD}$ if the attacker is the client (see \pref{fig:fl_proc} left). In VFL, the exploited information is referred to as $g_k^{\calD}$ (see \pref{fig:fl_proc} right). We unify them as $W_{k}^{\calD}$ and $W_{\text{fed}}^{\calD}$.
The utility of the federated model of HFL is usually defined as $\frac{1}{K}\sum_{k=1}^K U_k(\theta_\text{agg}^{\calD})$, while the one of VFL is $U(\theta_\text{joint}^{\calD})$. We unify them as $U(W_\text{fed}^{\calD})$.
\begin{table*}[!htp]%[tp]
  \centering
%   \small
  \setlength{\belowcaptionskip}{15pt}
   \vspace{-1em}
  \caption{Unified terminologies and notations for HFL and VFL}
  \label{tab: HFL_vs_VFL}
    \begin{tabular}{c||ccc}
    \toprule
    \hline
    & Horizontal FL & Vertical FL & Unified FL\cr
    % \midrule
    \cline{1-4}
    \begin{tabular}[c]{@{}c@{}} 
     global model
    \end{tabular} & 
    \begin{tabular}[c]{@{}c@{}} 
    \text{aggregated model}
    \end{tabular}
     & \begin{tabular}[c]{@{}c@{}}
     joint model 
     \end{tabular}
     & \begin{tabular}[c]{@{}c@{}}
     federated model 
     \end{tabular}
     \cr
     \cline{1-4}
    \begin{tabular}[c]{@{}c@{}} privacy \end{tabular} & 
    \begin{tabular}[c]{@{}c@{}} private data or\\ statistical information \end{tabular} 
    & private labels & 
    \begin{tabular}[c]{@{}c@{}} 
    private information
    \end{tabular} \cr
    \cline{1-4}
    % \begin{tabular}[c]{@{}c@{}}
    % information transmitted \\ by protector \end{tabular} & 
    % \begin{tabular}[c]{@{}c@{}} 
    % $W_k^{\calD}$ (protector: client) \\
    % or $W_{\text{fed}}^{\calD}$ (protector: server) \end{tabular}  &  $g_k^{\calD}$ & 
    % \begin{tabular}[c]{@{}c@{}} 
    % $W_{k}^{\calD}$ 
    % \end{tabular}\cr
    % \cline{1-4}
    \begin{tabular}[c]{@{}c@{}}
    information exposed \\to the attacker \end{tabular} & \begin{tabular}[c]{@{}c@{}}  
    $\theta_k^{\calD}$ (attacker: server) \\
    or $\theta_\text{agg}^{\calD}$ (attacker: client) \end{tabular}  
    &  $g_k^{\calD}$ %, \theta^{\calD}, \nabla\theta^{\calD}
    & 
    \begin{tabular}[c]{@{}c@{}}
    % information exposed to the protector\\
    $W_{k}^{\calD}$ or $W_\text{fed}^{\calD}$ 
    \end{tabular}\cr
    \cline{1-4}
    \renewcommand{\arraystretch}{1.8}
    \begin{tabular}[c]{@{}c@{}} 
    utility
    \end{tabular}
    & $\frac{1}{K}\sum_{k=1}^K U_k(\theta_\text{agg}^{\calD})$ & $U(\theta_\text{joint}^{\calD})$ & $U(W_\text{fed}^{\calD})$\cr
     \hline
    \bottomrule
    \end{tabular}
\end{table*}

Horizontal and vertical federated learning are the two primary forms of federated learning. While they have distinct training procedures, their privacy-preserving problem can be boiled down to the competition between the adversary and the protector. The former tries to learn as much private information on $D$ as possible via privacy attacks. At the same time, the latter applies protection mechanisms to mitigate privacy leakage while maintaining utility loss and efficiency reduction below an acceptable level.

\begin{figure*}[h!]
\centering
\includegraphics[width=0.98\linewidth]{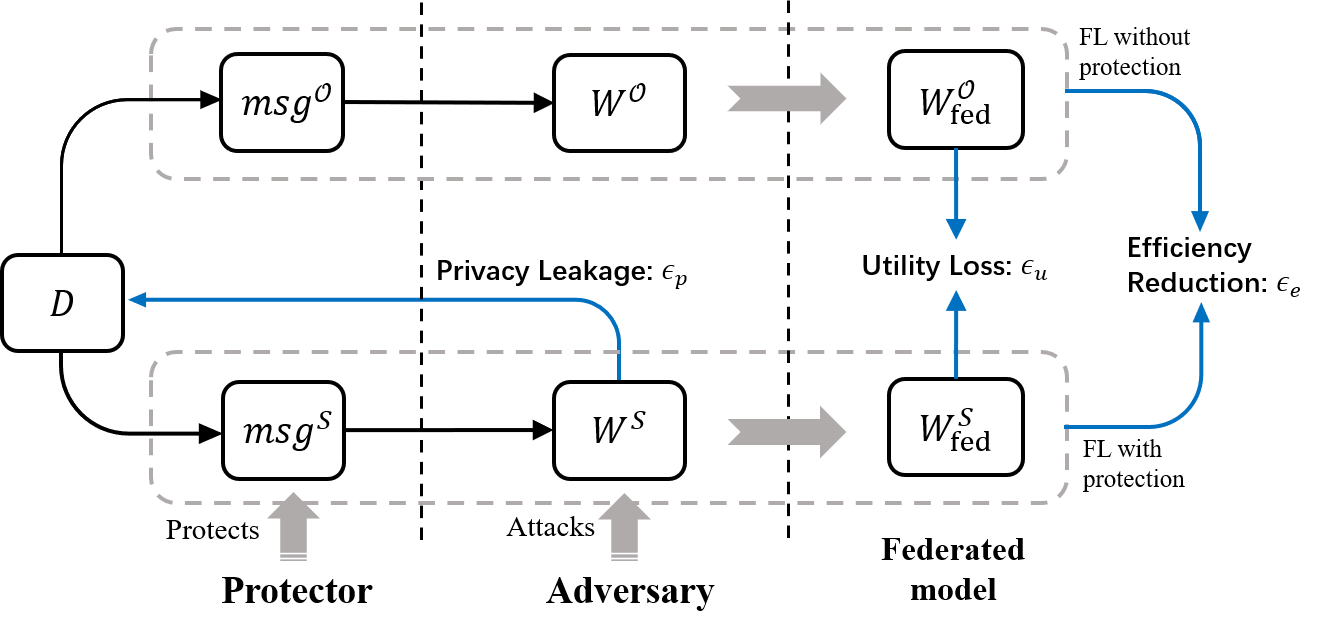}
\small
\vspace{-1em}
\caption{\textbf{The unified federated learning framework} (including HFL and VFL) illustrates the relationship between privacy leakage $\epsilon_p$, utility loss $\epsilon_u$, and efficiency reduction $\epsilon_e$. $D$ denotes the private data, which can be either features or labels. $msg$ denotes the message encoded with knowledge of $D$. It is sent from the data owner to the adversary. In HFL, $msg$ can be model parameters, model gradients, and model outputs sent from individual clients to the server. In VFL, $msg$ is typically the intermediate gradient (see 
Figure \ref{fig:fl_proc} (right)) sent from the client with labels to adversarial clients. $W$ denotes any model information derived from $msg$ that the privacy attack can leverage to infer $D$. In HFL, $W$ typically includes model parameters and model gradients. In VFL, $W$ includes intermediate gradients, model parameters, and model gradients. $W_{\text{fed}} $ denotes the federated model. We use superscripts $S$ and $\calRO$ to distinguish the protected information and the plain-text one, respectively.} \label{fig:pue}
\end{figure*}

Figure \ref{fig:pue} gives our proposed unified federated learning framework that illustrates the relationship between the \textit{privacy leakage} $\epsilon_p$, \textit{utility loss} $\epsilon_e$ and \textit{efficiency reduction} $\epsilon_u$ through the lens of the adversary and the protector. Specifically, to protect the privacy of its local data $D$, the protector converts the plain-text message $msg$ to protected one $msg^{\calD}$ exposed to the adversary aiming to mitigate the privacy leakage while maintaining the utility loss below an acceptable level. On the other hand, the adversary launches privacy attacks on protected model information $W^{\calD}$ aiming to infer as much information on $D$ as possible. Since protection mechanisms may jeopardize model utility and reduce efficiency, the protector needs to control the strength of the applied protection mechanism to strike a balance between private leakage, utility loss, and efficiency reduction. 

The unified FL framework provides a conceptual view of the relationship between $\epsilon_p$, $\epsilon_u$, and $\epsilon_e$ through the lens of the protector and the adversary, regardless of the underlying FL architecture. Thus, it also reconciles other FL architectures, such as peer-to-peer (P2P) FL. In this work, we propose the No-Free-Lunch (NFL) theorem that quantifies the trade-off between $\epsilon_p$, $\epsilon_u$, and $\epsilon_e$ from the perspective of the protector of our unified federated learning framework. Thus, the NFL theorem applies to HFL, VFL, and P2P FL (e.g., Swarm Learning~\cite{warnat2021swarm}).

\subsection{Protection and Privacy Attacking Mechanisms}

\subsubsection{Protection Mechanisms} The data protector applies certain protection mechanisms to the exposed model information to mitigate privacy leakage. We formally define the protection mechanism as follows:

\begin{definition}[Protection Mechanism and Protected Distribution]
The \textbf{protection mechanism} $M: \R^{m}\rightarrow \R^{m}$ maps the original model information $W_k^{\calRO}$, which follows a distribution $P_k^{\calRO}$, to its protected (or distorted) counterpart $W^{\calD}_k$, which follows a distribution $P_k^{\calD}$. The objective of $M$ is to protect private data so that the dependency between $W_k^{\calD}$ and $D_k$ is reduced, compared to the dependency between the unprotected information $W_k^{\calRO}$ and $D_k$. The distribution $P_k^{\calD}$ is referred to as the \textbf{protected distribution} of client $k$.
\end{definition}

\textbf{EXAMPLE:} We take the randomization mechanism as an illustrative example and introduce the \textit{protected distribution}. Assume that $W_k^{\calRO}\sim P_k^{\calRO} = \calN(\mu_0,\Sigma_0)$, and $\epsilon_k\sim \calN(0,\Sigma_{\epsilon})$, where $\Sigma_0 = \text{diag}(\sigma_{1}^2,\cdots, \sigma_{n}^2)$, $\Sigma_\epsilon = \text{diag}(\sigma_\epsilon^2, \cdots, \sigma_\epsilon^2)$. Then the protected parameter $W_k^{\calD} = W_k^{\calRO} + \epsilon_k\sim\calN(\mu_0, \Sigma_0+ \Sigma_\epsilon)$. That is, the protected distribution $P_k^{\calD} = \calN(\mu_0, \Sigma_0+ \Sigma_\epsilon)$. Please refer to \pref{sec: analysis_for_Randomization} for more details.

% \red{The protection mechanisms that we will analyze in detail include: \textbf{Randomization}, \textbf{Homomorphic Encryption}, \textbf{Secret Sharing} and \textbf{Compression}. Now we introduce the related concepts formally.}

% In this work, we consider widely-adopted protection mechanisms, including \textbf{Randomization Mechanism}, \textbf{Homomorphic Encryption}, \textbf{Secret Sharing} and \textbf{Compression Mechanism}. 

In this work, we consider widely-adopted protection mechanisms, including Randomization Mechanism, Homomorphic Encryption, Secret Sharing and Compression Mechanism. 

\subsubsection{Privacy Attacking Mechanisms}\label{App:bayes-inference-attack}

%In this section, we provide a general framework for semi-honest attacking mechanisms. 

The requirement of adopting specific protection mechanisms depends on the adversary's threat model. We consider the adversary to be \textit{semi-honest}, he/she faithfully adheres to the federated learning protocol, yet may execute privacy attacks on exposed data to deduce the private information of other participants.

%\begin{definition}[The Goal of the Adversary]\label{defi: defi_of_the_goal}
%\textbf{The Goal of the Adversary}
Let $D_k$ denote private features or labels in horizontal federated learning and denote labels in vertical federated learning. Let $W_k$ represent the model information exposed by client $k$. Let $f_{D_k|W_k}$ represent the probability density function of the posterior distribution $F_{D_k|W_k}$. Next, we introduce the Bayesian Inference Attack that estimates the private information by maximizing the posterior belief $f_{D_k|W_k}$.
%\red{The goal of the adversary is to \red{estimate} the posterior distribution $F_{D_k|W_k}$ of the private information $D_k$, given the exposed model information $W_k$.}

% \textbf{Competition between the protector and the adversary}
% Viewing from the goal of the protector and the adversary which are formally stated in \pref{eq: constraint_optimization_problem} (or \pref{eq: constrainted_optimization_problem_relaxed_with_efficiency_constraint}) and \pref{eq: the_goal_of_the_adversary}, the competition between the protector and the adversary is formally expressed as
% \begin{align}
%     \min_{P_k^{\calD}}\max_{F_{D_k|W_k}} \epsilon_{p} (P_k^{\calD}, F_{D_k|W_k}).
% \end{align}

% \textbf{Remark:}

% (1) For horizontal federated learning, $D$ represents the private data or the feature information, and $W$ represents the model information. 

% (2) For vertical federated learning, $D$ represents the label of $Y$, and $W$ represents the intermediate information.

% Let $f_{D_k|W_k}$ represent the probability density function of the posterior distribution $F_{D_k|W_k}$. Next, we introduce a family of attacking mechanisms referred to as the Bayesian Inference Attack that estimates the private information by maximizing the posterior belief $f_{D_k|W_k}$.

\begin{definition}[Bayesian Inference Attack]\label{defi: bayesian_inference_attack}
% The \textbf{Bayesian Inference Attack (BIA)} aims at inferring the private information upon observing the exposed information from the Bayesian perspective.

Given the distorted model information $W_k^{\calD}$, the \textbf{Bayesian Inference Attack} finds data $d$ that maximizes the posterior belief:
    \begin{align} \label{eq: optimizaiton}
    \arg\max\limits_{d} \text{IH}(d|w) = \arg\max\limits_{d}[\text{I}(w|d) + \text{H}(d)],
\end{align}
where $\text{IH}(d|w) = \log f_{{D_k}|{W_k^{\calD}}}(d|w)$ corresponds to the logarithm of the posterior belief, $\text{I}(w|d) = \log f_{W_k^{\calD}|D_k}(w|d)$ measures the logarithm of the likelihood based on the observed model information $w$, and $\text{H}(d) = \log f_{D_k}(d)$ represents the logarithm of the prior belief $f_{{D_k}}(d)$. According to Bayes' theorem, maximizing the logarithm of posterior $f_{D_k|W_k^{\calD}}(d|w)$ on $D_k$ involves maximizing summation of $\log(f_{W_k^{\calD}|D_k}(w|d))$ and $\log(f_{D_k}(d))$.
% Some of the representative attacking mechanisms for both horizontal FL and vertical FL fit into our proposed framework are summarized in \pref{tab: attacking_mechanisms_illustration} and elaborated as follows.
% The widely-used attacking mechanisms in both horizontal and vertical FL fit into our proposed framework, which are summarized in \pref{tab: attacking_mechanisms_illustration}.
\end{definition}
Below, we introduce some representative privacy attacks (see \pref{tab: attacking_mechanisms_illustration}) that fall under the family of the Bayesian Inference Attack.

\begin{table*}[!htp]%[tp]
  \centering
  %\begin{threeparttable}
  \small
%   \vspace{-1em}
  \caption{Privacy attacking mechanisms that fit the Bayesian inference attack. For Gradient Inversion attacks, $g$ denotes the observed gradient, and $w$ corresponds to $\theta_k^{\calD}$ in the left panel of Figure \ref{fig:fl_proc}. For Label Inference attacks, $w$ corresponds to $g_1^{\calD}$ in the right panel of Figure \ref{fig:fl_proc}).}
  \label{tab: attacking_mechanisms_illustration}
    \begin{tabular}{c||cccc}
    \toprule
    \hline
    Attack & Work & $\text{I}(w|d)$ & $\text{H}(d)$ & Type\cr
    \hline
    % \midrule
  \multirowcell{4}{Gradient \\Inversion} 
  & DLG \cite{zhu2019dlg} & $C - \frac{1}{2\sigma^2}\|g - \nabla_w \calL(d,w)\|_2^2$ & Constant & Horizontal\cr
  %\hline
  & Inverting Gradients \cite{geiping2020inverting} &  $\frac{\inner{\nabla_{w}\calL(d,w),\, g}}{\|\nabla_{w}\calL(d,w)\|\cdot \|g\|}$ & ${\text{TV}}(d)$ & Horizontal\cr
  & Improved DLG \cite{zhao2020idlg} & $C - \|g - \nabla_w \calL(d,w)\|_2^2$ & $\text{Label}(d)$ & Horizontal\cr
  & GradInversion \cite{yin2021see} & $C - \frac{1}{2\sigma^2}\|g - \nabla_w \calL(d,w)\|_2^2$ & $\text{Group}(d)$ & Horizontal\cr
%   \hline
%  \midrule[0.5pt]
     \hline
  \multirowcell{2} {Label \\ Inference} 
  & Norm-based Scoring~\cite{oscar2022split}& $N(\sum_{i = 1}^{\psi} \one\{r(w)\in [l_i, u_i]\}\cdot \one\{d\in C_i\})$ &  Constant & Vertical \cr
  & Direct Label Inference~\cite{fu2022label}& $N(\sum_{i = 1}^{\psi} \one\{w\in [l_i, u_i]\}\cdot \one\{d\in C_i\})$ & Constant & Vertical\cr
     \hline
    \bottomrule
    \end{tabular}
\end{table*}

\paragraph{Privacy Attacks in HFL}
The following attacking mechanisms proposed for HFL fit into the Bayesian inference attack framework formulated in \pref{eq: optimizaiton}. These attacks infer private data $d$ through maximizing the similarity between the observed gradient $g$ and the estimated gradient $\nabla_w \calL(d,w)$ with different prior on $d$ ($w$ here corresponds to $\theta_k^{\calD}$ in the left panel of Figure \ref{fig:fl_proc}).
\begin{itemize}
    \item Deep Gradient Leakage (DLG)~\cite{zhu2019dlg}: $\text{H}(d)$ is constant; $\text{I}(w|d)$ is the negative $\ell_2$ distance between the observed gradient $g$ and the estimated gradient $\nabla_w \calL(d,w)$.
    \item Inverting Gradients~\cite{geiping2020inverting}: $\text{H}(d)$ is the {\text{TV}} loss of estimated data, denoted as ${\text{TV}}(d)$; $\text{I}(w|d)$ corresponds to the cosine similarity between the observed gradient $g$ and the estimated gradient $\nabla_w \calL(d,w)$.
    \item Improved DLG~\cite{zhao2020idlg}): $\text{H}(d)$ is the prior with the label information of $d$, denoted as $\text{Label}(d)$; $\text{I}(w|d)$ corresponds to the negative $\ell_2$ distance between the observed gradient $g$ and the estimated gradient $\nabla_w \calL(d,w)$.
    %\item DeepInversion (\cite{yin2020dreaming}). 
    \item GradInversion~\cite{yin2021see}: $\text{H}(d)$ is the group consistency of estimated data, denoted as $\text{Group}(d)$; $\text{I}(w|d)$ corresponds to the negative $\ell_2$ distance between the observed gradient $g$ and the estimated gradient $\nabla_w \calL(d,w)$.
\end{itemize}
% where $\nabla W(d)$ is the gradients of training loss w.r.t. model parameters for the estimated data $D$, and $H(d)=\log f_{D_k}(d)$.

\paragraph{Privacy Attacks in VFL} We focus on label inference attacks because labels may contain sensitive user information, and they are typically valuable assets in real-world VFL applications such as finance, healthcare, and advertisement. More specifically, we consider gradient scoring attack and direct label inference attack, which comply with the Bayesian inference attack framework. Both attacks are mounted by client 1, aiming to infer labels owned by client 2 based on back-propagated gradient $w$ ($w$ here corresponds to $g_1^{\calD}$ in the right panel of Figure \ref{fig:fl_proc}).
\begin{itemize}
    % \item Norm-based Scoring Attack (NBS)~\cite{oscar2022split}. 
    %  NBS  observes that $\|w\|_2$ of the positive instances are generally larger than that of the negative ones and formulates the attacking problem as the classification problem, which could also be regarded as the optimization problem formulated in \pref{eq: optimizaiton}. Let $\text{I}(w|d) = N\left(\sum_{i = 1}^{\psi} \one\{r(w)\in [l_i, v_i]\}\cdot \one\{d\in C_i\}\right)$, where $r: \R^m \rightarrow [0,1]$ denotes a scoring function; $l_i$ and $v_i$ denotes the lower bound and upper bound, respectively; $\psi$ represents the total number of classes; $C_i$ represents class $i$; $N(\cdot)$ computes the logarithm of the normalized input. $\text{H}(d) = \text{constant}$. The goal of the attacker is to design a scoring function (classifier) satisfying that $r(w)\in [l_1, v_1]$ if the corresponding data belongs to the negative class, and $r(w)\in [l_2, v_2]$ if the corresponding data belongs to the positive class.
    \item Gradient Scoring Attack. $\text{I}(w|d) = N\left(\sum_{i = 1}^{\psi} \one\{r(w)\in [l_i, v_i]\}\cdot \one\{d\in C_i\}\right)$, where $r: \R^m \rightarrow [0,1]$ denotes a scoring function; $d$ represents one data point; $l_i$ and $v_i$ denotes the lower bound and upper bound, respectively; $\psi$ represents the total number of classes; $C_i$ represents class $i$; $N(\cdot)$ computes the logarithm of the normalized input. $\text{H}(d) = \text{constant}$. The gradient scoring attack typically applies to the binary classification task. The attacker needs to design a scoring function satisfying that $r(w)\in [l_1, v_1]$ if the corresponding data of $w$ belongs to the negative class, and $r(w)\in [l_2, v_2]$ if the corresponding data of $w$ belongs to the positive class.
    A case in point is the Norm-based Scoring Attack (NBS)~\cite{oscar2022split}, which observes that $\|w\|_2$ of the positive instances are generally larger than that of the negative ones and formulates the attacking problem as the classification problem.
    % which could also be regarded as the optimization problem formulated in \pref{eq: optimizaiton}.
    \item Direct Label Inference Attack (DLI) \cite{fu2022label}. 
    % The adversary \blue{(i.e., client 1 in Figure \ref{fig:fl_proc})} uses the exposed gradients to infer the labels of client 2. 
    DLI is tailored to the VFL scenario where the top model owned by client 2 is an activation function (e.g., softmax)~\cite{liu2022vertical}, the adversary has access to the gradients of the final activation function, and thus it can infer labels through the signs of back-propagated gradients. In DLI, $\text{I}(w|d) = N\left(\sum_{i = 1}^{\psi} \one\{w_i\in ([l_i, v_i]\}\cdot \one\{d\in C_i\}\right)$, where $l_i =-\infty$ and $v_i = 0$ if $d\in C_i$; $l_i = 0$ and $v_i=+\infty$ otherwise; $\text{H}(d) = \text{constant}$. Thus, if $w_i < 0$ its corresponding data point $d$ belongs to class $C_i$ and if $w_i >0$ the data point $d$ does not belong to class $C_i$.

    % where $d$ represents one data point, $\psi$ represents the total number of classes, $C_i$ represents class $i$, $N(\cdot)$ computes the logarithm of the normalized input, $w_i = \frac{\partial \text{loss} (\hat{c},c)}{\partial u_i^{\text{adv}}}$ ($i$-th dimension of $w$ corresponding to $g_1^{\calD}$ in the right panel of Figure \ref{fig:fl_proc})  represents the gradient of the loss w.r.t. the $i$-th logit $u_1$ from the adversary transmitted from the server to the adversary, $\hat{c}$ represents the estimated label and $c$ represents the ground-truth label and $y_i^{\text{adv}}$ represents the logits for class $i$. 
\end{itemize}

\textbf{Remark:} Attacking mechanisms, such as the model completion attack \cite{fu2022label}, that use the cumulative information over rounds and is beyond the scope of our article.

% \begin{definition}[Model Training]
% Several gradients need to be calculated, including
% \begin{itemize}
%     \item gradient of the loss $L$ with respect to the logit
%     \begin{align*}
%         \frac{\partial L}{\partial l} = \tilde p_1 - y,
%     \end{align*}
%     \item Calculate the gradient with respect to the cut layer outputs (communicate this information back to the non-label party \red{to allow the non-label party to compute the gradients of its parameters})
%     \begin{align*}
%         g = \nabla_{f(X)}L = (\tilde p_1 - y)\nabla_z h(z)|_{z = f(X)}\in\R^d.
%     \end{align*}
% \end{itemize}
% \end{definition}

\subsection{Privacy Leakage, Utility Loss and Efficiency Reduction}

% We formally define \textit{Bayesian privacy leakage}, \textit{utility loss} and \textit{efficiency reduction} (depicted in \pref{fig:pue}). The privacy leakage (\pref{defi: average_privacy_JSD}) assesses the variation in the adversaries' beliefs with and without leaked information under the Bayesian inference attack (\pref{defi: bayesian_inference_attack}). The posterior belief is averaged with respect to the protected model information that is exposed to adversaries. 

% The original model information $W_k^{\calRO}$ is modified as its protected counterpart $^{\calD}_k$, resulting in a model with lower accuracy and efficiency in certain scenarios. The incurred utility loss (\pref{defi: utility_loss}) and efficiency reduction (\pref{defi: efficiency_reduction}) are defined as the difference of utility and efficiency with and without protections.

%\blue{
In this section, we formally define \textit{privacy leakage}, \textit{utility loss}, and \textit{efficiency reduction} (depicted in \pref{fig:pue}). 

The privacy leakage (\pref{defi: average_privacy_JSD}) measures the variation between the adversary's prior and posterior beliefs on private information. The adversary obtains the posterior belief by mounting a Bayesian inference attack (\pref{defi: bayesian_inference_attack}) on the protected model information. Thus, we formally call the privacy leakage the \textit{Bayesian} privacy leakage. The protected model information may lead to the federated model with lower utility (\pref{defi: utility_loss}) and the federated training with less efficiency (\pref{defi: efficiency_reduction}) in certain scenarios.
%}
% \red{The incurred utility loss (\pref{defi: utility_loss}) and efficiency reduction (\pref{defi: efficiency_reduction}) are defined as the difference of utility and efficiency with and without the protection.}}

Let $F^{\calA}_k$, $F^{\mathcal O}_k$ and $F^{\calO}_k$ represent the attacker's belief distributions about $D_k$ upon observing the protected information, the original information and without observing any information, respectively, and the probability density functions of which are $f^{\calA}_{D_k}, f^{\calRO}_{D_k}$, and $f^{\calO}_{D_k}$. Specifically, $f^{\calA}_{D_k}(d) = \int_{\mathcal{W}_k} f_{{D_k}|{W_k}}(d|w)dP^{\calD}_{k}(w)$, $f^{\calRO}_{D_k}(d) = \int_{\mathcal{W}_k}f_{{D_k}|{W_k}}(d|w)dP^{\calRO}_{k}(w)$, and $f^{\calO}_{D_k}(d) = f_{D_k}(d)$.

We use {\text{JS}} divergence to measure the privacy leakage instead of KL divergence. The advantage of {\text{JS}} divergence over KL divergence is that it is symmetrical, and its square root satisfies the triangle inequality \cite{endres2003new}. This property facilitates the quantification of the trade-offs.

\begin{definition}[Bayesian Privacy Leakage]\label{defi: average_privacy_JSD}
%The average Bayesian privacy leakage is defined as
Let $\epsilon_{p,k}$ represent the privacy leakage of client $k$, which is defined as:
\begin{align}\label{eq: def_of_pl}
&\epsilon_{p,k} = \sqrt{{\text{JS}}(F^{\calA}_k || F^{\calO}_k)},
\end{align}
where ${\text{JS}}(F^{\calA}_k || F^{\calO}_k) = \frac{1}{2}\int_{\mathcal{D}_k} f^{\calA}_{D_k}(d)\log\frac{f^{\calA}_{D_k}(d)}{f^{\calM}_{D_k}(d)}\textbf{d}\mu(d) + \frac{1}{2}\int_{\mathcal{D}_k} f^{\calO}_{D_k}(d)\log\frac{f^{\calO}_{D_k}(d)}{f^{\calM}_{D_k}(d)}\textbf{d}\mu(d)$, $F^{\calA}_k$ and $F^{\calO}_k$ represent the attacker's belief distribution about $D_k$ upon observing the protected information and without observing any information, respectively, and $f_{D_k}^{\calM}(d) = \frac{1}{2}(f^{\calA}_{D_k}(d) + f^{\calO}_{D_k}(d))$.

Furthermore, the Bayesian privacy leakage in FL resulting from releasing the protected model information is defined as:
\begin{align*}
\epsilon_p = \frac{1}{K}\sum_{k=1}^K \epsilon_{p,k}.
\end{align*} 
\textbf{Remark:}\\
(1) The local model information $W_k$ represents the \textit{model parameters}, \textit{model gradients} and \textit{model outputs}, all of which may optionally be exchanged or get exposed to semi-honest adversaries.
\end{definition}
\noindent(2) If the private information $D_k$ is continuous, then 
\begin{align*}
   {\text{JS}}(F^{\calA}_k || F^{\calO}_k) 
    & = \frac{1}{2}\left[\text{KL}\left(F^{\calA}_k, F^{\calM}_k\right) + \text{KL}\left(F^{\calO}_k,F^{\calM}_k\right)\right]\\
    & = \frac{1}{2}\int_{\mathcal{D}_k} f^{\calA}_{D_k}(d)\log\frac{f^{\calA}_{D_k}(d)}{f^{\calM}_{D_k}(d)}\textbf{d}\mu(d) + \frac{1}{2}\int_{\mathcal{D}_k} f^{\calO}_{D_k}(d)\log\frac{f^{\calO}_{D_k}(d)}{f^{\calM}_{D_k}(d)}\textbf{d}\mu(d).
\end{align*}

If the private information $D_k$ is discrete, then 
\begin{align*}
   {\text{JS}}(F^{\calA}_k || F^{\calO}_k) 
   & = \frac{1}{2}\left[\text{KL}\left(F^{\calA}_k, F^{\calM}_k\right) + \text{KL}\left(F^{\calO}_k,F^{\calM}_k\right)\right]\\
   & = \frac{1}{2}\sum_{d\in\mathcal{D}_k} f^{\calA}_{D_k}(d)\log\frac{f^{\calA}_{D_k}(d)}{f^{\calM}_{D_k}(d)} + \frac{1}{2}\sum_{d\in\mathcal{D}_k} f^{\calO}_{D_k}(d)\log\frac{f^{\calO}_{D_k}(d)}{f^{\calM}_{D_k}(d)}.
\end{align*}
%Without loss of generality, we use the continuous version for analysis. 

%(2) The Bayesian privacy leakage measures the discrepancy between the adversaries' belief with and without leaked information $W_k^{\calD}$ for $k_{th}$ client. Moreover, the Bayesian privacy leakage is averaged with respect to the protected model information variable $W^{\calD}_k$ which is exposed to adversaries.\\ %Moreover, the maximizing Bayesian privacy leakage could be defined as $\max_{W_k^{\calD}}P_{D_k|F^{\calO}_k, W_k^{\calD}}$, which reflects the maximum privacy leakage over all $W_k^{\calD}$.\\
% (3) In particular, when $W_d$ is independent of $D$, $P_{S|F^{\calO}_k, W_d} = P_{S|F^{\calO}_k}$, which lea\textbf{d}\mu(d) the Bayesian privacy leakage to be zero. \\

%\subsection{Utility Loss in Secure Federated Learning}

When evaluating the utility loss, we consider the scenario when each protector is assigned a private key (if one exists).

%\blue{
\begin{definition}[Utility Loss]\label{defi: utility_loss}
%\blue{
The utility loss of client $k$ (denoted as $\epsilon_{u,k}$) measures the variation in utility of client $k$ with the federated model drawn from unprotected distribution $P_{\text{fed}}^{\calRO}$ and the utility of the federated model drawn from protected distribution $P_{\text{fed}}^{\calD}$:
%}
\begin{align*}
    \epsilon_{u,k} = \mathbb E_{W_{\text{fed}}^{\calRO}\sim P_{\text{fed}}^{\calRO}}[U_k(W^{\calRO}_{\text{fed}})] - \mathbb E_{W_{\text{fed}}^{\calD}\sim P_{\text{fed}}^{\calD}}[U_k(W_{\text{fed}}^{\calD})],
\end{align*}
where $U_k$ represents the utility function of client $k$. Furthermore, the utility loss in FL system is defined as:
\begin{align*}
    \epsilon_u = \mathbb E_{W_{\text{fed}}^{\calRO}\sim P_{\text{fed}}^{\calRO}}[U(W^{\calRO}_{\text{fed}})] - \mathbb E_{W_{\text{fed}}^{\calD}\sim P_{\text{fed}}^{\calD}}[U(W_{\text{fed}}^{\calD})],
\end{align*}
where $U$ represents the utility function of the FL system.  
%, the expectation is taken with respect to the parameters
%$v$ represents the information set associated with the attacker, including rounds for attacking, public key, and the scoring function, which is problem-dependent.
\end{definition}
%}

% \textbf{Remark:}
% (1) The model utility evaluates model performance for a variety of learning tasks. For example, model utility is classification accuracy in a classification problem and prediction accuracy in a regression problem.\\
% (2) For randomization mechanism, the model utility measured by \cite{li2021label} measure utility using $p\cdot \text{tr}(\text{Cov}[\eta^{(1)}]) + (1-p)\cdot \text{tr}(\text{Cov}[\eta^{(0)}])$, where $p$ represents the fraction of positive examples; $\text{tr}(\text{Cov}[\eta^{(i)}])$ represents the trace of the covariance matrix of the random noise $\eta^{(i)}$ and measures the noise level. 

%The noise level $\text{tr}(\text{Cov}[\eta^{(i)}])$ of class $i$ is weighted by its proportion, thereby controlling the gradient noise of each training example equally.

% Efficiency reduction is a critical criteria that is also explored. We define it as follows.

Efficiency reduction is another criterion we consider in this work, and defined as follows.

%\blue{
\begin{definition}[Efficiency Reduction]\label{defi: efficiency_reduction}
Let $\epsilon_{e,k}$ represent the efficiency reduction of client $k$. The efficiency reduction of client $k$ measures the variation in efficiency with the models drawn from the unprotected and protected distributions $P_{k}^{\calRO}$ and $P_{k}^{\calD}$, which is defined as: 
%The efficiency reduction measures the difference between the efficiency of all participating models drawn from their corresponding protected distributions $P_{k}^{\calD}$ and the efficiency of those models drawn from the unprotected distributions $P_{k}^{\calRO}$: 
% \begin{align*}
%     \epsilon_e = \frac{1}{K}\sum_{k = 1}^K \left[\mathbb E_{W_k^{\calD}\sim P_{k}^{\calD}}[C(W^{\calD}_k)] - \mathbb E_{W_k^{\calRO}\sim P_{k}^{\calRO}}[C(W^{\calRO}_k)]\right],
% \end{align*}

\begin{align*}
    \epsilon_{e,k} = \mathbb E_{W_k^{\calD}\sim P_{k}^{\calD}}[C(W^{\calD}_k)] - \mathbb E_{W_k^{\calRO}\sim P_{k}^{\calRO}}[C(W^{\calRO}_k)],
\end{align*}
where $C$ denotes a mapping from the model information to the efficiency measured in terms of the communication cost (e.g., the transmitted bits) or the training cost. Furthermore, the efficiency reduction in FL system is defined as:
\begin{align*}
    \epsilon_e = \frac{1}{K}\sum_{k = 1}^K \epsilon_{e,k}.
\end{align*}

%where $C$ represents the communication cost.\\
% \textbf{Remark:} The efficiency was measured using communication bits in \cite{chen2020breakin}.
\end{definition}

%}

\subsection{The Competition Between The Protector and The Adversary}
% \blue{
% \textbf{The Goal of the Protector}
% The goal of the protector can be formulated as a constrained optimization problem. Given a privacy level of $\chi_k$, the goal of protector $k$ is to find a protected distribution $P_k^{\calD}$ that achieves the least amount of utility loss and efficiency reduction while adhering to the privacy constraint. The objective function is formally expressed as
% \begin{align} \label{eq: constraint_optimization_problem}
% \min\limits_{P_k^{\calD}} \eta_{u}\cdot\epsilon_{u,k} +  \eta_{e}\cdot\epsilon_{e,k},
% \end{align}
% where $\eta_{u}$ represents the preference of the protector towards model utility, and $\eta_{e}$ represents the preference of the protector towards efficiency. The constraint is expressed as $\epsilon_{p,k}\le\chi_k$ for the chosen protected distribution $P_k^{\calD}$.
% }

\textbf{The Goal of the Protector}
% The goal of the protector can be formulated as a constrained optimization problem. The protector aims at minimizing the payoff measured using utility loss and privacy leakage under the privacy constraint. We provide an example in \pref{eq: constraint_optimization_problem}. 
The goal of protector $k$ is formulated as an optimization problem that aims at finding a protected distribution $P_k^{\calD}$ achieving the minimum utility loss and efficiency reduction under the privacy constraint $\chi_k$, which is formally expressed as
\begin{align} \label{eq: constraint_optimization_problem}
\min\limits_{P_k^{\calD}} \eta_{u}\cdot\epsilon_{u,k} +  \eta_{e}\cdot\epsilon_{e,k},\\
\text{subject to} \quad \epsilon_{p,k}\le\chi_k.
\end{align}
where $\eta_{u}$ represents the preference of the protector towards model utility, and $\eta_{e}$ represents the preference of the protector towards efficiency.

% \textbf{Remark:} This optimization problem can be extended to satisfy personalized requirement of the federated learning system. For example, if the efficiency reduction is required not to exceed $\varphi_k$, then the optimization problem is expressed as
% \begin{align} \label{eq: constrainted_optimization_problem_relaxed_with_efficiency_constraint}
% \begin{array}{r@{\quad}l@{}l@{\quad}l}
% \quad\min\limits_{P_k^{\calD}}&\epsilon_{u,k},\\
% \text{s.t.,} & \epsilon_{p,k}\le\chi_k, \epsilon_{e,k}\le \varphi_k.\\
% \end{array}
% \end{align}

%\end{definition}

\textbf{The Goal of the Adversary}
Let $D_k$ be the client $k$'s private information. Let $W_k^{\calD}$ represent the client $k$'s model information exposed to the adversary. The goal of the adversary is formally expressed as
\begin{align}\label{eq: the_goal_of_the_adversary}
    \max_{F_{D_k|W_k^{\calD}}} \epsilon_{p,k} := {\text{JS}}(F^{\calA}_k || F^{\calO}_k),
\end{align}
where $F_{D_k|W_k^{\calD}}$ represents the adversary's posterior distribution of the private information $D_k$ upon observing the exposed model information $W_k^{\calD}$.

\textbf{Competition Between the Protector and the Adversary}
Given the goal of the protector and the adversary formally stated in \pref{eq: constraint_optimization_problem} and \pref{eq: the_goal_of_the_adversary}, the competition between the protector and the adversary is formally expressed as 
\begin{align} \label{eq: constraint_optimization_problem_add_max}
\begin{array}{r@{\quad}l@{}l@{\quad}l}
\quad\min\limits_{P_k^{\calD}}& \eta_{u}\cdot\epsilon_{u,k} +  \eta_{e}\cdot\epsilon_{e,k},\\
\text{subject to} & \max_{F_{D_k|W_k^{\calD}}} \epsilon_{p,k}\le\chi_k.
\end{array}
\end{align}
where $\chi_k$ is the privacy level required by client $k$. 

\textbf{Remark:} This optimization problem can be extended to satisfy the personalized requirement of the federated learning system. For example, if the efficiency reduction of the federated learning system is required not to exceed $\varphi_k$, then the optimization problem is expressed as
\begin{align} \label{eq: constrainted_optimization_problem_relaxed_with_efficiency_constraint}
\begin{array}{r@{\quad}l@{}l@{\quad}l}
\quad\min\limits_{P_k^{\calD}}&\epsilon_{u,k},\\
\text{subject to} & \max_{F_{D_k|W_k^{\calD}}} \epsilon_{p,k}\le\chi_k, \epsilon_{e,k}\le \varphi_k.\\
\end{array}
\end{align}

A natural question is: \textit{is it possible to design a protection mechanism that simultaneously achieves infinitesimal privacy leakage, utility loss, and efficiency reduction?} In the following section, we provide a negative answer for this question in certain scenarios, which leads to our No-Free-Lunch Theorem (\pref{thm: utility-privacy-efficiency Trade-off_JSD_mt}).

%% file: 06_Trade_off.tex
\section{No Free Lunch Theorem in Federated Learning}\label{sec: trade_offs}

% In this section, we first use the total variation distance between the protected and unprotected distributions as an intermediate quantity to quantify utility loss, privacy leakage and efficiency reduction separately. Our major result is illustrated in \pref{thm: utility-privacy-efficiency Trade-off_JSD_mt}, in which we quantify the trade-off between privacy leakage, utility loss and efficiency reduction in certain scenarios. This states that it is unrealistic to expect an algorithm to simultaneously provide excellent privacy, utility, and efficiency in some scenarios. 

% We first elaborate on the scenarios that guarantee the existence of trade-off between utility and privacy. The following assumption states that the cumulative density of the near-optimal parameters is bounded, which rules out scenarios in which the utility function is constant or impossible to distinguish between the near-optimal parameters and the sub-optimal parameters. This assumption is presented in \cite{zhang2022no}, we demonstrate it here to be thorough and show it in \pref{fig: delta_w}.

%\blue{
In this section, we propose the No-Free-Lunch (NFL) theorem (\pref{thm: utility-privacy-efficiency Trade-off_JSD_mt}), stating that it is unrealistic to expect a privacy-preserving FL algorithm to simultaneously achieve the infinitesimal privacy leakage, utility loss, and efficiency reduction in certain scenarios. 

% Before elaborating on our NFL theorem, we explain the scenarios that guarantee the existence of trade-offs between utility and privacy. \pref{assump: assump_of_Delta} states that the cumulative density of the near-optimal parameters is bounded, which rules out scenarios where the utility function is constant or impossible to distinguish between the near-optimal and sub-optimal parameters. This assumption is proposed in \cite{zhang2022no}. We present it here for completeness and illustrate it in \pref{fig:delta_w}.

Before elaborating on the NFL theorem, we introduce \pref{assump: assump_of_Delta} and \pref{assump: assump_of_Xi_efficiency}, under which the federated learning scenarios we analyze guarantee the existence of trade-offs between privacy, utility, and efficiency.

\pref{assump: assump_of_Delta} is proposed in \cite{zhang2022no} and illustrated in \pref{fig:delta_w}. We present it here for completeness. Intuitively, it states that the cumulative density of the near-optimal parameters is bounded, which rules out scenarios where the utility is constant and 
% it is impossible to distinguish between the near-optimal and sub-optimal parameters.
most parameters are near-optimal parameters. 
%the near-optimal and sub-optimal parameters can not be distinguishable. }

\begin{assumption}\label{assump: assump_of_Delta}
let $\calW^{\calD}_{\text{fed}}$ represent the support of the protected distribution $P_{\text{fed}}^{\calD}$ of the federated model information. Let $\calW^{*}_{\text{fed}}$ represent the set of parameters achieving the maximum utility
% That is, the optimal utility parameters
\begin{align*}
    \calW^{*}_{\text{fed}} = \argmax_{w\in\calW_{\text{fed}}}U(w),
\end{align*}
where $U$ represents the utility of the federated learning system. Given a non-negative constant $c$, the \textit{near-optimal parameters} is defined as
% $$\calW_{c} = \left\{w\in\calW^{\calD}_{\text{fed}}: \left|U(w^{*})-U(w)\right|\le c, \forall w^{*}\in\calW^{*}_{\text{fed}}\right\}.$$
$$\calW_{c} = \left\{w\in\calW^{\calD}_{\text{fed}}: \left|U(w^{*})-U(w)\right|\le c, \forall w^{*}\in\calW^{*}_{\text{fed}}\right\}.$$
Let $c = \Delta$ be the \textit{maximum} constant that satisfies:
\begin{align}\label{eq: Delta_Area}
     \int_{\calW_{\ttwo{fed}}^{\calD}} \one\{w\in\calW_{\Delta}\} p^{\calD}_{W_{\text{fed}}}(w)  dw\le\frac{{\text{TV}}(P_{\text{fed}}^{\calRO} || P_{\text{fed}}^{\calD})}{2}, 
\end{align}

% \begin{align}
%     \underbrace{\int_{\mathcal U_{k}}\one\{w\in\calW_{k,\Xi_k}^{+}\} p^{\calD}_{W_{k}}(w)dw}_{\text{term $1$}} - \underbrace{\int_{\mathcal U_{k}}\one\{w\in\calW_{k,\Xi_k}^{-}\}p^{\calD}_{W_{k}}(w)dw}_{\text{term $2$}}\ge\Gamma_k\cdot{\text{TV}}(P_{k}^{\calRO} || P_{k}^{\calD}),
% \end{align}

where $p^{\calD}_{W_{\text{fed}}}$ represents the probability density function of the protected federated model information $W_{\text{fed}}^\calD$. We assume that $\Delta$ is positive, i.e., $\Delta >0$.
\end{assumption}
% \red{draw a picture}\\
\begin{figure*}[!htp]
\centering
\includegraphics[width = 0.55\columnwidth]{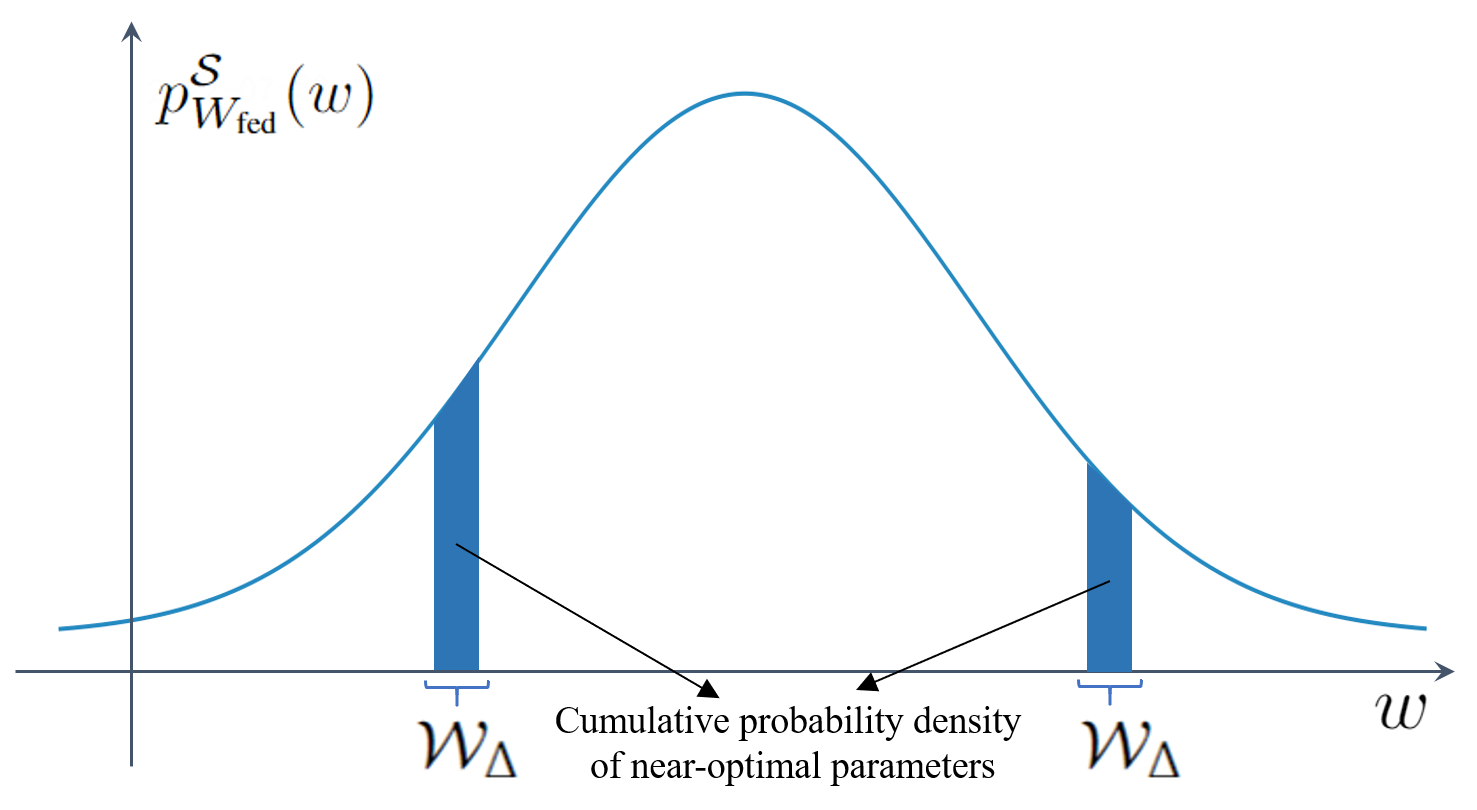}
\vspace{-1em}
\caption{Illustration of \pref{assump: assump_of_Delta}. The blue area represents the probability that the realization of the protected model information $W_{\text{fed}}^\calD$ falls inside the \textit{near-optimal parameters} $\calW_{\Delta}$. It is referred to as the cumulative probability density of the near-optimal parameters. Note that we use model parameters with one dimension in this figure for illustrative purposes.\\
}\label{fig:delta_w}
\end{figure*}

%This figure illustrates that the sub-optimal parameters \red{dominate} the near-optimal parameters $\calW_{\Delta}$ in terms of the cumulative probability density.

%\red{Add intuitive explanation for Assumption 4.1 by Xiaojin}

% The colored area corresponds to the LHS of \pref{eq: Delta_Area}.

\textbf{Remark:} The Eq. (\ref{eq: Delta_Area}) states that  the cumulative density of the near-optimal distorted parameters is upper-bounded by $\frac{{\text{TV}}(P_{\text{fed}}^{\calRO} || P_{\text{fed}}^{\calD} )}{2}$, which is at most $\frac{1}{2}$. Therefore, the utility is not constant, and most distorted parameters are sub-optimal.

\textbf{Remark:}
If the distribution of the protected model information is continuous and ${\text{TV}}(P^{\calRO}_{\text{fed}} || P^{\calD}_{\text{fed}} ) = 0$, then there does not exist a positive constant $\Delta$ satisfying \pref{eq: Delta_Area}, which implies that \pref{assump: assump_of_Delta} does not hold. From Lemma C.3 of \cite{zhang2022no}, the utility equals $0$.

%\red{do we have an example for this remark?}\\

%Now we introduce the scenarios guarantee that the existence of efficiency reduction

% To analyze the trade-off between privacy leakage and efficiency reduction, we assume that the cumulative density of the distorted parameters with a lower efficiency is guaranteed to be greater than those with a higher efficiency, which further guarantees the existence of efficiency reduction (being positive) (see \pref{assump: assump_of_Xi_efficiency}). 

% \blue{The following \pref{assump: assump_of_Xi_efficiency} states that the cumulative density of the distorted parameters with higher communication cost is larger than that of the distorted parameters with smaller communication cost, which rules out scenarios where the communication cost is constant and the communication cost of the majority of the distorted parameters is smaller than that of the original parameters.
% }

% \blue{The following \pref{assump: assump_of_Xi_efficiency} states that the probability density of the distorted parameters is larger than that of original parameters and distorted parameters have higher communication cost than that of the original parameters, which rules out scenarios where the communication cost is constant, and the communication cost of the majority of the distorted parameters is smaller than that of the original parameters. 
% }

Let $p_{W_k}^{\calRO}$ represent the probability density of unprotected model information of client $k$, and $p_{W_k}^{\calD}$ represent the probability density of protected model information of client $k$. The following \pref{assump: assump_of_Xi_efficiency} states that the distorted parameters have higher communication cost than that of the original parameters with high probability. Therefore, the expected communication cost of the distorted parameter is higher than that of the original parameter. This assumption rules out scenarios where the communication cost is constant, and the communication cost of most distorted parameters is smaller than that of the original parameters. We provide an example for illustration in \pref{fig:assump_2}.

\begin{assumption}\label{assump: assump_of_Xi_efficiency}
Let $w^{\calRO}_{\text{max}} = \arg\max_{w\in\calW_k^{\calRO}} C(w)$. Let $\Xi_k$ denote the minimum non-negative constant satisfying that $C(w^{\calRO}_{\text{max}}) - \Xi_k \le C(w)$, $\forall w\in\mathcal W^{\calD}_{k}$. Let $\mathcal U_{k}$ represent the set of distorted parameters with improved probability density. Specifically,  
$\mathcal U_{k} = \{w\in\mathcal W_{k}^{\calD}: p_{W_{k}}^{\calD}(w) \ge p^{\calRO}_{W_{k}}(w)\}$. We denote $\calW_{k,\Xi_k}^{+}$ as the set of distorted parameters with high probability density, and the communication cost of which is larger than that of the original parameters with a gap of at least $2 \cdot \Xi_k$. We denote $\calW_{k,\Xi_k}^{+}$ as:
\begin{align*}
&\calW_{k,\Xi_k}^{+} = \left\{w\in\mathcal U_{k}: C(w^{\calRO}_{\text{max}}) + 2 \cdot \Xi_k \le C(w) \quad\text{and}\quad p_{W_{k}}^{\calD}(w) \ge 2 \cdot p^{\calRO}_{W_{k}}(w)\right\}. 
\end{align*}

% and the probability density of these distorted parameters are larger than that of the original parameters

We denote $\calW_{k,\Xi_k}^{-} = \mathcal U_{k}\setminus\calW_{k,\Xi_k}^{+}$. Let $\Gamma_k$ denote the maximum constant satisfying that
% \begin{align}
%     \int_{\mathcal U_{k}}\left[\underbrace{\one\{w\in\calW_{k,\Xi_k}^{+}\} p^{\calD}_{W_{k}}(w)}_{\term_1} - \underbrace{\one\{w\in\calW_{k,\Xi_k}^{-}\}p^{\calD}_{W_{k}}(w)}_{\term_2}\right] dw\ge\Gamma_k\cdot{\text{TV}}(P_{k}^{\calRO} || P_{k}^{\calD} ),
% \end{align}

\begin{align} \label{eq:assump_2}
    \underbrace{\int_{\mathcal U_{k}}\one\{w\in\calW_{k,\Xi_k}^{+}\} p^{\calD}_{W_{k}}(w)dw}_{\text{term $1$}} - \underbrace{\int_{\mathcal U_{k}}\one\{w\in\calW_{k,\Xi_k}^{-}\}p^{\calD}_{W_{k}}(w)dw}_{\text{term $2$}}\ge\Gamma_k\cdot{\text{TV}}(P_{k}^{\calRO} || P_{k}^{\calD} ),
\end{align}
where $p^{\calD}_{W_{k}}$ denotes the probability density function of the protected model information and $p^{\calRO}_{W_{k}}$ denotes the probability density function of the original model information.
We assume that $\Xi_k$ and $\Gamma_k$ are positive, i.e., $\Xi_k > 0, \Gamma_k > 0, \forall k\in[K]$.
\end{assumption}

\begin{figure*}[!htp]
\centering

\includegraphics[width = 0.38\linewidth]{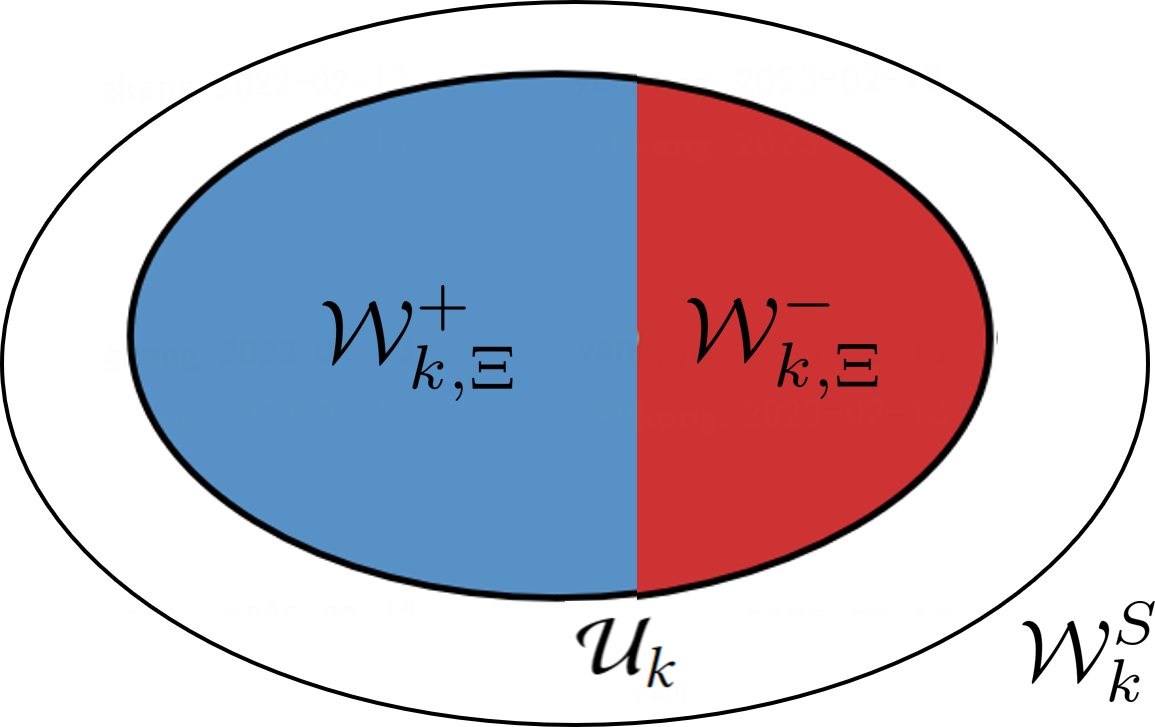}
\hspace{2em}
\includegraphics[width = 0.48\linewidth]{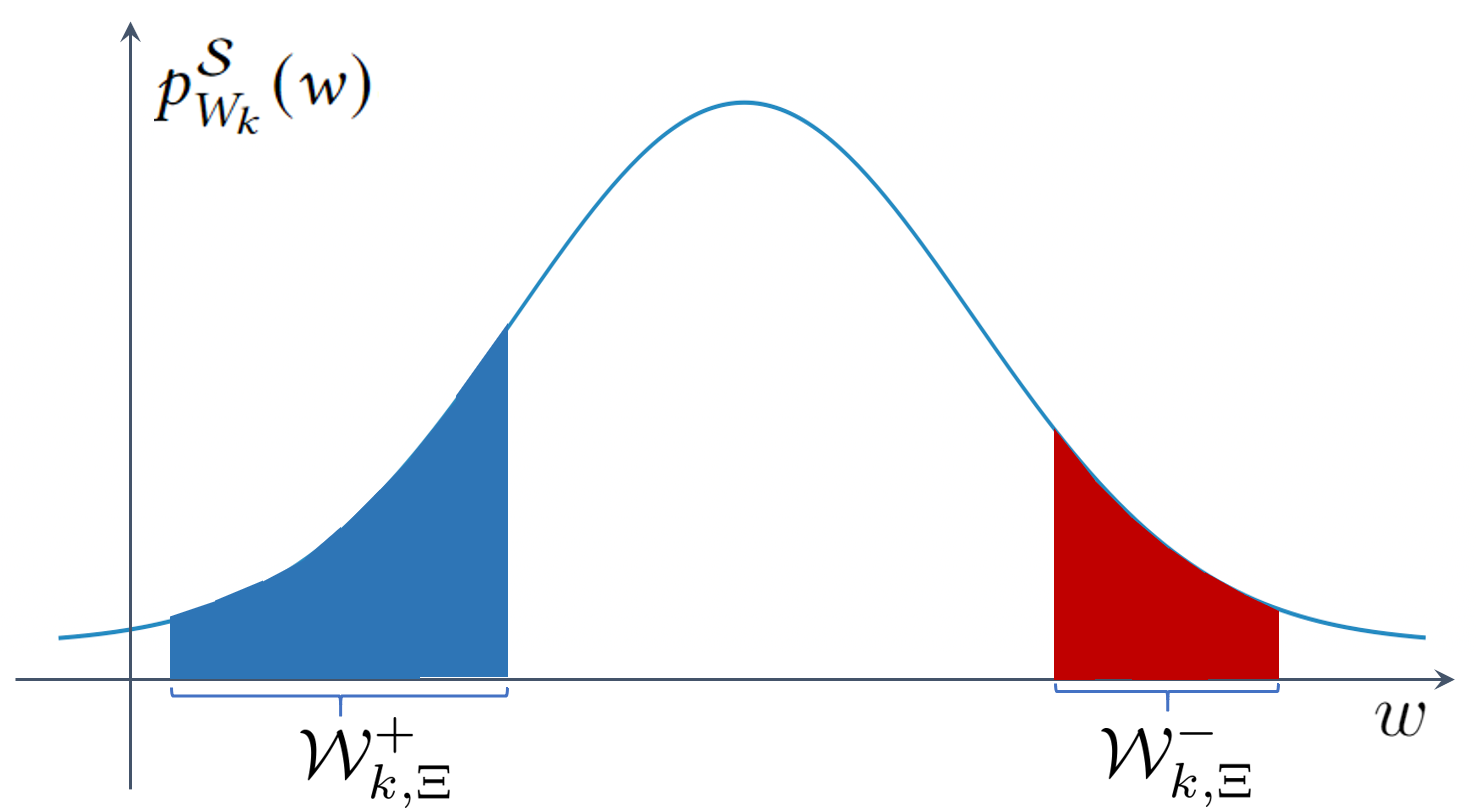}

\caption{Illustration of \pref{assump: assump_of_Xi_efficiency}. The left panel shows the relationship between the sets $\calW_{k,\Xi_k}^{+}$, $\calW_{k,\Xi_k}^{-}$, $\mathcal U_{k}$, and $\mathcal W^{\calD}_{k}$. The right panel illustrates Eq. (\ref{eq:assump_2}). In the right panel, the blue area (\text{term $1$} in Eq.(\ref{eq:assump_2})) and red area  (\text{term $2$} in Eq.(\ref{eq:assump_2})) represent the probability that the realization of the protected model information falls inside $\calW_{k,\Xi_k}^{+}$ and $\calW_{k,\Xi_k}^{-}$, respectively, and the blue area is larger than the red area. Note that we use model parameters with one dimension for illustrative purposes.
}\label{fig:assump_2}
\end{figure*}

%\red{Add intuitive explanation for Assumption 4.2 by Xiaojin}

\textbf{Remark:} The Eq. (\ref{eq:assump_2}) states that the difference between the distorted parameters' cumulative density in $\mathcal{W}^{+}_{k,\Xi_k}$ and $\mathcal{W}^{-}_{k,\Xi_k}$ is lower-bounded by a non-negative value. Therefore, it conveys that the expected communication cost of the distorted parameter is higher than that of the original parameter. Since $\Xi_k > 0$, the communication cost is not constant.

In the next section, we introduce the quantitative trade-offs between privacy and efficiency and between privacy and utility, respectively. These two trade-offs lead to the main conclusion of our No-Free-Lunch theorem formulated in \pref{thm: utility-privacy-efficiency Trade-off_JSD_mt}.

%\red{It is better have intuitive explanation for $C_2$, $C_d$, and $C_x$.}
\subsection{Trade-off between Privacy and Efficiency}

The following Lemma \ref{lem: efficiency_reduction_and_tvd} bounds efficiency reduction using the distortion measured by total variation distance ${\text{TV}}(P^{\calRO}_{\text{fed}} || P^{\calD}_{\text{fed}} )$. The intuition is that a larger distortion leads to a higher efficiency reduction. The analysis is deferred to \pref{lem: total_variation-efficiency trade-off}.

\begin{lem}\label{lem: efficiency_reduction_and_tvd}
Let \pref{assump: assump_of_Xi_efficiency} hold, and $\epsilon_{e}$ be defined in \pref{defi: efficiency_reduction}. Let $P_{\text{fed}}^{\calRO}$ and $P_{\text{fed}}^{\calD}$ represent the distribution of the aggregated parameter before and after being protected. Let $\Xi = \min_{k\in[K]} \Xi_k$, and $\Gamma = \min_{k\in [K]}\Gamma_k$. Then, we have:
\begin{align*}
    \epsilon_{e} \ge\Xi\cdot\Gamma\cdot \frac{1}{K}\sum_{k = 1}^K {\text{TV}}(P^{\calRO}_{k} || P^{\calD}_{k} ).
\end{align*}
\end{lem}

% Intuitively, a larger distortion would result in higher computational complexity, and lower privacy leakage. The following lemma illustrates that the summation of the privacy leakage and the efficiency reduction is lower bounded by a problem-dependent constant. 

The following Lemma \ref{lem: total_variation-privacy trade-off} illustrates the quantitative relationship between $\epsilon_{p}$ and total variation distance. 
% \red{When the private information $D$ is continuous, the result is proven in \cite{zhang2022no}; When the private data $D$ is discrete, the same result is shown in \pref{lem: total_variation-privacy trade-off_app_1}. The comprehensive result, which is applicable to both discrete and continuous private information, is shown in the following lemma}. 
Lemma \ref{lem: total_variation-privacy trade-off} applies to both discrete (see proof in \cite{zhang2022no}) and continuous private information (see  \pref{lem: total_variation-privacy trade-off_app_1}).  

\begin{lem}\label{lem: total_variation-privacy trade-off}
Let $\epsilon_{p,k}$ be defined in \pref{defi: average_privacy_JSD}, $P_k^{\calRO}$ and $P^{\calD}_k$ represent the distribution of the parameter of client $k$ before and after being protected. Then, we have:
\begin{align}\label{eq: total_variation-privacy trade-off}
\epsilon_p \ge C_1 - \frac{1}{K}\sum_{k=1}^K C_2\cdot {\text{TV}}(P_k^{\calRO} || P^{\calD}_k),
\end{align}
where $C_1 = \frac{1}{K}\sum_{k=1}^K \sqrt{{\text{JS}}(F^{\calRO}_k || F^{\calO}_k)}$, and $C_2 = \frac{1}{2}(e^{2\xi}-1)$, $\xi_k$=$\max_{w\in \mathcal{W}_k, d \in \mathcal{D}_k} \left|\log\left(\frac{f_{D_k|W_k}(d|w)}{f_{D_k}(d)}\right)\right|$, and $\xi$=$\max_{k\in [K]} \xi_k$ represents the maximum privacy leakage over all possible information $w$ released by the clients, and $[K] = \{1,2,\cdots, K\}$.
\end{lem}

Intuitively, a larger distortion would increase communication costs while decreasing privacy leaks. The Lemma \ref{thm: efficiency-privacy trade-off_JSD_mt} demonstrates how the summation of privacy leakage and efficiency reduction is lower bounded by a constant that depends on the nature of the scenario.

\begin{lem}[No free lunch theorem (NFL) for privacy and efficiency]\label{thm: efficiency-privacy trade-off_JSD_mt}
Let $\epsilon_p$ be defined in Def. \ref{defi: average_privacy_JSD}, and let $\epsilon_e$ be defined in Def. \ref{defi: efficiency_reduction}, with \pref{assump: assump_of_Xi_efficiency} we have:
\begin{align*}
    \epsilon_{p} + C_x\cdot\epsilon_e\ge C_1,
\end{align*}
where $C_1 = \frac{1}{K}\sum_{k=1}^K \sqrt{{\text{JS}}(F^{\calRO}_k || F^{\calO}_k)}$ and $C_x = \frac{1}{2\Xi\Gamma}(e^{2\xi}-1)$.
\end{lem}

\subsection{Trade-off between Privacy and Utility}

The Lemma \ref{lem: total_variation-utility trade-off_mt} demonstrates how the total variation distance between the protected and unprotected distributions lower bounds utility loss.

\begin{lem}\label{lem: total_variation-utility trade-off_mt}
Let \pref{assump: assump_of_Delta} hold, and $\epsilon_{u}$ be defined in \pref{defi: utility_loss}. Let $P_{\text{fed}}^{\calRO}$ and $P_{\text{fed}}^{\calD}$ represent the distribution of the federated model information before and after being protected. Then, we have:
\begin{align*}
    \epsilon_{u} \ge& \frac{\Delta}{2}\cdot {\text{TV}}(P^{\calRO}_{\text{fed}} || P^{\calD}_{\text{fed}} ).
\end{align*}
\end{lem}

With \pref{lem: total_variation-privacy trade-off} and \pref{lem: total_variation-utility trade-off_mt}, it is now natural to provide a quantitative relationship between the utility loss and the privacy leakage.

\begin{lem}[No free lunch theorem (NFL) for privacy and utility]\label{lem: utility-privacy trade-off_JSD_mt} 
Let $\epsilon_p$ be defined in Def. \ref{defi: average_privacy_JSD}, and let $\epsilon_u$ be defined in Def. \ref{defi: utility_loss} at the convergence step, with \pref{assump: assump_of_Delta} we have:
\begin{align}\label{eq: total_variation-privacy trade-off_app_2}
 \epsilon_{p} + C_d\cdot \epsilon_{u}\ge C_1,
\end{align}
where $\xi_k$=$\max_{w\in \mathcal{W}_k, d \in \mathcal{D}_k} \left|\log\left(\frac{f_{D_k|W_k}(d|w)}{f_{D_k}(d)}\right)\right|$, and $\xi$=$\max_{k\in [K]} \xi_k$ represents the maximum privacy leakage over all possible information $w$ released by client $k$, and $[K] = \{1,2,\cdots, K\}$, $C_1 = \frac{1}{K}\sum_{k=1}^K \sqrt{{\text{JS}}(F^{\calRO}_k || F^{\calO}_k)}$, $C_d = \frac{\gamma}{4\Delta}(e^{2\xi}-1)$, and $\gamma = \frac{\frac{1}{K}\sum_{k=1}^K {\text{TV}}(P_k^{\calRO} || P^{\calD}_k)}{{\text{TV}}(P^{\calRO}_{\text{fed}} || P^{\calD}_{\text{fed}} )}$.
\end{lem}
\textbf{Remark:}
If the distribution of the protected model information is continuous, and ${\text{TV}}(P^{\calRO}_{\text{fed}} || P^{\calD}_{\text{fed}} ) = 0$, then \pref{assump: assump_of_Delta} does not hold, and the utility loss is equal to $0$ (from Lemma C.3 of \cite{zhang2022no}).

\subsection{Trade-off between Privacy, Utility and Efficiency}\label{subsect:PU-tradeoff}

The Theorem \ref{thm: utility-privacy-efficiency Trade-off_JSD_mt} illustrates the quantitative trade-off between privacy leakage, utility loss, and efficiency reduction, which shows that the weighted summation of these three metrics is larger than a problem-dependent constant (also exhibited in \pref{fig: pue_tradeoff}). It implies that simultaneously achieving infinitesimal privacy leakage, utility loss, and efficiency reduction is unrealistic in certain scenarios. The complete analysis is deferred to Appendix \ref{sec: trade_off_privacy_utility_efficiency}.

\begin{thm}[No free lunch theorem (NFL) for privacy, utility and efficiency]\label{thm: utility-privacy-efficiency Trade-off_JSD_mt} 
Let $\epsilon_e$ be defined in Def. \ref{defi: efficiency_reduction}, $\epsilon_p$ be defined in Def. \ref{defi: average_privacy_JSD}, and let $\epsilon_u$ be defined in Def. \ref{defi: utility_loss} at the convergence step. Let \pref{assump: assump_of_Delta} and \pref{assump: assump_of_Xi_efficiency} hold, then we have that:
\begin{align}\label{eq: total_variation-privacy trade-off_mt}
\epsilon_{p} + \frac{C_d}{2}\cdot \epsilon_{u} + \frac{C_x}{2}\cdot \epsilon_{e}\ge C_1,
\end{align}
where $\xi_k = \max_{w\in \mathcal{W}_k, d \in \mathcal{D}_k} \left|\log\left(\frac{f_{D_k|W_k}(d|w)}{f_{D_k}(d)}\right)\right|$, $\xi = \max_{k\in [K]} \xi_k$, $C_1 = \frac{1}{K}\sum_{k=1}^K \sqrt{{\text{JS}}(F^{\calRO}_k || F^{\calO}_k)}$, $C_d = \frac{\gamma}{4\Delta}(e^{2\xi}-1)$, $C_x = \frac{1}{2\Xi\Gamma}(e^{2\xi}-1)$, $\gamma = \frac{\sum_{k=1}^K {\text{TV}}(P_k^{\calRO} || P^{\calD}_k)}{{\text{TV}}(P^{\calRO}_{\text{fed}} || P^{\calD}_{\text{fed}} )}$.

\end{thm}

\begin{figure}[!htp]
\centering
\includegraphics[width = 0.65\columnwidth]{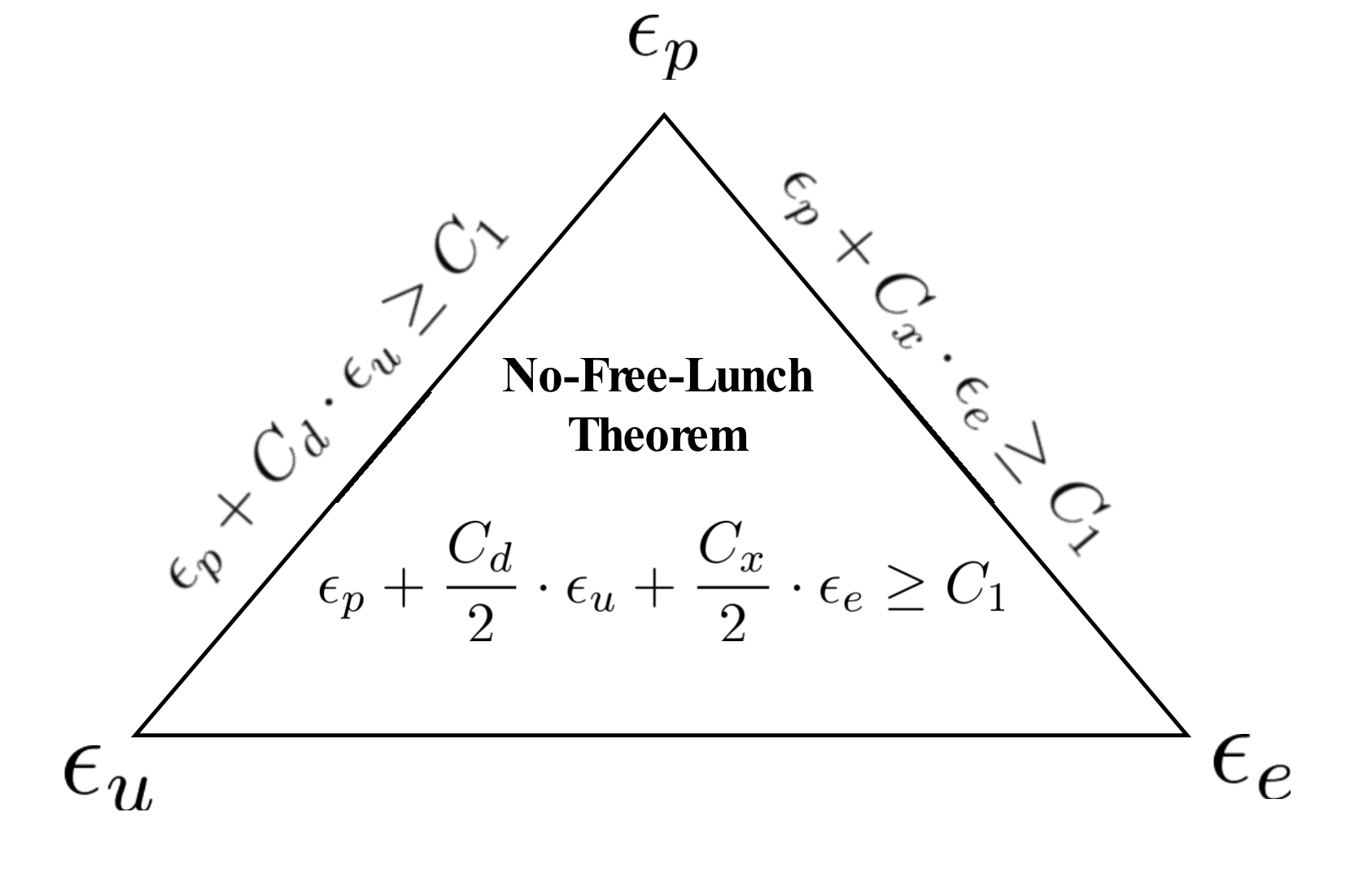}
\vspace{-1em}
\caption{Privacy-Utility-Efficiency Trade-offs}
\label{fig: pue_tradeoff}
\end{figure}

Note that $C_1$ measures the discrepancy between the prior belief and the posterior belief of the attacker and is typically a positive constant in real-world applications. \pref{thm: utility-privacy-efficiency Trade-off_JSD_mt} states that the weighted summation of the privacy leakage, utility loss, and efficiency reduction is greater than or equal to $C_1$ in certain scenarios, indicating that a protection mechanism cannot be expected to achieve exceptional privacy, utility, and efficiency simultaneously in scenarios where \pref{assump: assump_of_Delta} and \pref{assump: assump_of_Xi_efficiency} hold.

%% file: 07_Applications.tex
\section{Applications}\label{sec: applications}
% In this section, we apply the proposed No-Free-Lunch theorem for privacy leakage, utility loss and efficiency reduction trade-off to four protection mechanisms, considering the scenario when the efficiency is measured in terms of the communication cost. We denote $\delta>0$ as a small constant, $C_1 = \frac{1}{K}\sum_{k=1}^K \sqrt{{\text{JS}}(F^{\calRO}_k || F^{\calO}_k)}, C_2 =(e^{2\xi}-1)/2$ are two constants independent of the protection mechanisms adopted, and $m$ is the dimension of the parameter.

In this section, we apply the No-Free-Lunch theorem to four well-adopted protection mechanisms for quantifying their trade-offs between privacy leakage, utility loss, and efficiency reduction.

We measure efficiency in terms of the communication cost and denote $\delta>0$ as a small constant. $C_1 = \frac{1}{K}\sum_{k=1}^K \sqrt{{\text{JS}}(F^{\calRO}_k || F^{\calO}_k)}$ and $C_2 =(e^{2\xi}-1)/2$ are two constants independent of the protection mechanisms adopted, and $m$ is the dimension of the parameter.

\begin{table*}[htbp]
\vspace{-1em}
\caption{Comparison of homomorphic encryption, secret sharing, randomization, and compression in terms of privacy-utility-efficiency trade-off characterized via lower bounds for $\epsilon_p$, $\epsilon_u$, and $\epsilon_e$. LB denotes lower bound.}
\label{tab:tradeoff-comp}
\scriptsize
\begin{tabular}{c||c|c|c|c}
\hline
% \begin{tabular}[c]{@{}c@{}}Protection \\Mechanism\end{tabular} 
& \begin{tabular}[c]{@{}c@{}}Paillier HE\end{tabular} 
&
\begin{tabular}[c]{@{}c@{}} Secret Sharing\end{tabular}
&
\begin{tabular}[c]{@{}c@{}} Randomization\end{tabular}
&
\begin{tabular}[c]{@{}c@{}} Compression\end{tabular}
\\ \hline
\begin{tabular}[c]{@{}c@{}} $\epsilon_p (\text{LB})$ \end{tabular}
 &\begin{tabular}[c]{@{}c@{}}  
\scriptsize{ $C_1 - C_2\cdot\left[ 1 - \left(\frac{2\delta}{n^2}\right)^m\right]$ }
 \end{tabular}
 & 
 \begin{small}
 \begin{tabular}[c]{@{}c@{}} \scriptsize{$C_1 - \frac{C_2}{K}\sum\limits_{k=1}^K \left( 1 - \prod\limits_{j = 1}^{m}\frac{2\delta}{b_k^j + r_k^j}\right)$} \end{tabular}
 \end{small}
 &\begin{tabular}[c]{@{}c@{}} \scriptsize{$1 - \frac{3 C_2}{2}\cdot\min\left\{1,\sigma_{\epsilon}^2\sqrt{\sum\limits_{i=1}^{m}\frac{1}{\sigma_{i}^4}}\right\}$} \end{tabular} 
 &\begin{tabular}[c]{@{}c@{}} \scriptsize{$C_1 - C_2\cdot \left(1 - \prod\limits_{i = 1}^{m}\rho_i\right)$} \end{tabular}
 \\ \hline
\begin{tabular}[c]{@{}c@{}} $\epsilon_u (\text{LB})$ \end{tabular}
& \begin{tabular}[c]{@{}c@{}}  \fillcell{$0$} \\ \fillcell{(\pref{assump: assump_of_Delta}} \\
\fillcell{does not hold)}
\end{tabular}
&\begin{tabular}[c]{@{}c@{}} \fillcell{$0$} \\ \fillcell{(\pref{assump: assump_of_Delta}}\\
\fillcell{does not hold)}
\end{tabular}
&\begin{tabular}[c]{@{}c@{}} $\frac{\Delta}{100}\cdot\min\left\{1,\sigma_{\epsilon}^2\sqrt{\sum\limits_{i=1}^{m}\frac{1}{\sigma_{i}^4}}\right\}$ \end{tabular}
&\begin{tabular}[c]{@{}c@{}}  $\frac{\Delta}{2}\cdot \left(1 - \prod\limits_{i = 1}^{m}\rho_i\right)$ \end{tabular}
\\ \hline
\begin{tabular}[c]{@{}c@{}} $\epsilon_e (\text{LB})$ \end{tabular}
& \begin{tabular}[c]{@{}c@{}} $\Xi\cdot\Gamma\cdot \left[ 1 - \left(\frac{2\delta}{n^2}\right)^m\right]$ \end{tabular} 
& \begin{tabular}[c]{@{}c@{}} 
\fillcell{N/A} \\
\fillcell{(\pref{assump: assump_of_Xi_efficiency}}\\
\fillcell{does not hold)}
\end{tabular} 
& \begin{tabular}[c]{@{}c@{}} 
$\frac{\Xi\cdot\Gamma}{100}\min\left\{1, \sigma_\epsilon^2\sqrt{\sum\limits_{i=1}^{m}\frac{1}{\sigma_i^4}} \right\}$
\end{tabular}
& \begin{tabular}[c]{@{}c@{}}
$\Xi\cdot\Gamma\cdot \left(1 - \prod\limits_{i = 1}^{m}\rho_i\right)$
\end{tabular}
\\ \hline
\end{tabular}
\end{table*}

\subsection{Randomization Mechanism}
Let $W_k^{\calRO}$ be the parameter sampled from distribution $P_k^{\calRO} = \calN(\mu_0,\Sigma_0)$, where $\mu_0 \in \mathbb{R}^n$, $\Sigma_0 = \text{diag}(\sigma_{1}^2,\cdots, \sigma_{m}^2)$ is a diagonal matrix. The distorted parameter $W_k^{\calD} = W_k^{\calRO} + \epsilon_k$, where $\epsilon_k \sim \calN(0, \Sigma_\epsilon)$ and $\Sigma_\epsilon = \text{diag}(\sigma_\epsilon^2, \cdots, \sigma_\epsilon^2)$. The bounds for privacy leakage, utility loss and efficiency reduction for randomization mechanism is illustrated in the following theorem. The full proof is deferred to Appendix \ref{sec: analysis_for_Randomization}.

\begin{thm}\label{thm: randomization_thm_1}
For randomization mechanism, the privacy leakage is bounded by
\begin{align*}
    \epsilon_p\ge C_1 - \frac{C_2}{100}\cdot\min\left\{1, \sigma_\epsilon^2\sqrt{\sum_{i=1}^{m}\frac{1}{\sigma_i^4}} \right\}. 
\end{align*}
The utility loss is bounded by
\begin{align*}
    \epsilon_u\ge\frac{\Delta}{200}\cdot\min\left\{1, \sigma_\epsilon^2\sqrt{\sum_{i=1}^{m}\frac{1}{\sigma_i^4}} \right\}. 
\end{align*}
The efficiency reduction is bounded by
\begin{align*}
    \epsilon_e\ge\frac{\Xi\Gamma}{100}\min\left\{1, \sigma_\epsilon^2\sqrt{\sum_{i=1}^{m}\frac{1}{\sigma_i^4}} \right\}.
\end{align*}
% where $C_1 = \frac{1}{K}\sum_{k=1}^K \sqrt{{\text{JS}}(F^{\calRO}_k || F^{\calO}_k)}, C_2 =(e^{2\xi}-1)/2$ are two constants independent of the protection mechanisms adopted, and $m$ is the dimension of the parameter.
\end{thm}

%\subsection{Homomorphic Encryption (HE)}
% \textbf{Centralized}
% Private data are encrypted and used to train machine learning models.  

% These data can be abused by the centralized data holder even though they are encrypted. Therefore, centralized HE based methods might incure security risks \cite{}

% \textbf{Decentralized}

\subsection{Paillier Homomorphic Encryption}

The \textbf{Paillier} encryption mechanism proposed by \cite{paillier1999public} is an asymmetric additive homomorphic encryption mechanism, which was widely applied in FL \cite{zhang2019pefl, aono2017privacy, truex2019hybrid, cheng2021secureboost, fang2021privacy}. Paillier encryption contains three parts: key generation, encryption, and decryption. Let $h$ represent the plaintext, and $c$ represent the ciphertext. Let ($n,g$) represent the public key, and ($\lambda, \mu$) represent the private key. Note that the primes $p$ and $q$ are rather large.

\paragraph{Encryption} Randomly select $r$ and encode $h$ as:
\begin{align*}
    c = g^h\cdot r^n \text{ mod } n^2,
\end{align*}
where $n = p \cdot q$, $p$ and $q$ are two selected primes. Note that $g$ is an integer selected randomly, and $g\in \mathbb Z_{n^2}^*$. Therefore, $n$ can divide the order of $g$.\\ 

\paragraph{Decryption} Using the private key $(\lambda, \mu)$ to decrypt the ciphertext $c$ as:
\begin{align*}
    h = L(c^{\lambda}\text{ mod }n^2)\cdot \mu \text{ mod } n,
\end{align*}
where $L(x) = \frac{x-1}{n}$, $\mu = (L(g^{\lambda}\text{ mod }n^2))^{-1}\text{ mod }n$, $\lambda = lcm (p-1, q-1)$, and $lcm$ represents the least common multiple.

% the distribution of which corresponds to a degenerate distribution. To calculate the total variation distance between the unprotected distribution (discrete) and the protected distribution (continuous), we relax the discrete unprotected distribution to a continuous distribution using a small parameter $\delta$.
Let $W_k^{\calRO} = (a_k^1, \cdots, a_k^{m})$ represent the plaintext parameter that follows a degenerate distribution. Let $W_k^{\calD}$ represent the ciphertext parameter that follows a uniform distribution over $[0,n^2-1]^{m}$. To facilitate the calculation of the total variation distance between the unprotected distribution of $W_k^{\calRO}$ and the protected distribution of $W_k^{\calD}$, we relax the discrete unprotected distribution to a continuous distribution using a parameter $\delta > 0$. Assuming that $W_k^{\calRO}$ follows a uniform distribution over $[a_k^1 - \delta, a_k^1 + \delta]\times [a_k^2 - \delta, a_k^2 + \delta]\times\cdots\times [a_k^{m} - \delta, a_k^{m} + \delta]$, we can calculate the total variation distance between the distribution of $W_k^{\calRO}$ and the distribution of $W_k^{\calD}$. The following \pref{thm: epsilon_p_and_one_over_n} provides bounds for privacy leakage and efficiency reduction using the magnitude of the ciphertext $n$. The bound for privacy leakage decreases with $n$, which could guide the selection of $n$ to adapt to privacy and efficiency requirements. The full proof is deferred to Appendix \ref{sec: analysis_for_Paillier}.

\begin{thm}\label{thm: epsilon_p_and_one_over_n}
For Paillier algorithm, the utility loss is $0$, the privacy leakage is bounded by:
\begin{align*}
    \epsilon_p \ge C_1 - C_2\cdot\left[ 1 - \left(\frac{2\delta}{n^2}\right)^m\right].
\end{align*}
The efficiency reduction is bounded by 
\begin{align*}
    \epsilon_e \ge \Xi\cdot\Gamma\cdot\left[1 - \left(\frac{2\delta}{n^2}\right)^{m}\right]. 
\end{align*}
\end{thm}

% \begin{theorem}
% For Paillier, the trade-off between privacy leakage and efficiency reduction is expressed as 
% \begin{align*}
%     \epsilon_p\ge \sqrt{{\text{JS}}(F^{\calO}_k || F^{\calRO}_k)} - \frac{1}{2}(e^{2\xi}-1)(1 - \frac{\delta}{g(\epsilon_e)}).
% \end{align*}

% \end{theorem}

%\subsection{Multi-Party Computation (MPC)}

\subsection{Secret Sharing Mechanism}
Various privacy-preserving protocols based on MPC (especially secret sharing) have been proposed to build secure machine learning models, including linear regression, logistic regression, and recommendation systems. \cite{SecShare-Adi79,SecShare-Blakley79,bonawitz2017practical} were proposed to distribute secrets between participants.

Note that ${\text{TV}}(P^{\calRO}_{\text{fed}} || P^{\calD}_{\text{fed}} ) = 0$ for secret sharing mechanism. From Lemma C.3 of \cite{zhang2022no}, the utility loss equals $0$. The communication cost for the parameter of the secret sharing mechanism is not guaranteed to satisfy \pref{assump: assump_of_Xi_efficiency}, and the analysis of the lower bound of the efficiency reduction is beyond the scope of our article.
%For secret sharing mechanism, \pref{assump: assump_of_Xi_efficiency} does not hold, so we do not provide analysis for efficiency reduction. 

Let $W_k^{\calRO}$ represent the original model information that follows a uniform distribution over $[a_k^1 - \delta, a_k^1 + \delta]\times [a_k^2 - \delta, a_k^2 + \delta]\cdots\times [a_k^{m} - \delta, a_k^{m} + \delta]$. Let $W_k^{\calD}$ represent the distorted model information that follows uniform distribution over $[a_k^1 - b_k^1, a_k^1 + r_k^1]\times [a_k^2 - b_k^2, a_k^2 + r_k^2]\cdots\times [a_k^{m} - b_k^{m}, a_k^{m} + r_k^{m}]$.  The following \pref{thm: secret_sharing} measures utility loss and provides the lower bound for privacy leakage. The full proof is deferred to Appendix \ref{sec: analysis_for_Secret_Sharing}.

\begin{thm}\label{thm: secret_sharing}
For secret sharing mechanism, the utility loss $\epsilon_u = 0$, the privacy leakage is bounded by
\begin{align*}
    \epsilon_p \ge C_1 - C_2\cdot\frac{1}{K}\sum_{k=1}^K \left( 1 - \prod_{j = 1}^{m}\frac{2\delta}{b_k^j + r_k^j}\right).
\end{align*}
\end{thm}

%% file: 08-Conclusion.tex
\subsection{Compression Mechanism}

%The trade-off between utility and efficiency was investigated by \cite{wang2018cooperative, wang2019adaptive, nori2021fast}. 

For the compression mechanism, the client does not transfer all the parameters to the server. Let $W_k^{\calRO}(i)$ ($W_k^{\calD}(i)$) denote dimension $i$ of $W_k^{\calRO}$ ($W_k^{\calD}$). Let $b_i$ represent a random variable that follows the Bernoulli distribution. Specifically, $b_i$ takes the value $1$ with probability $\rho_i$, and $0$ with probability $1 - \rho_i$. Each dimension $i$ of the distorted parameter $W_k^{\calD}(i)$ is defined as
\begin{equation*}
W_k^{\calD}(i) =\left\{
\begin{array}{cl}
 W_k^{\calRO}(i) &  \text{if } b_i = 1,\\
0,  &  \text{if } b_i = 0.\\
\end{array} \right.
\end{equation*}
The following \pref{thm: compression} provides bounds for privacy leakage, utility loss and efficiency reduction using the compression probability. The full proof is deferred to Appendix \ref{sec: analysis_for_Compression}.

% \subsubsection{Privacy-Utility-Efficiency Trade-off for Compression Mechanism}

\begin{thm} \label{thm: compression}
For compression mechanism, the privacy leakage is lower bounded by
\begin{align*}
    \epsilon_p\ge C_1 - C_2\cdot \left(1 - \prod_{i = 1}^{m} \rho_i\right).
\end{align*}
The utility loss is bounded by 
\begin{align*}
    \epsilon_u \ge \frac{\Delta}{2}\cdot \left(1 - \prod_{i = 1}^{m}\rho_i\right).
\end{align*}
The efficiency reduction $\epsilon_e$ is bounded by
\begin{align*}
    \epsilon_e\ge \Xi\cdot\Gamma\cdot \left(1 - \prod_{i = 1}^{m}\rho_i\right).
\end{align*}
\end{thm}

\section{Conclusion and Future Work}

In this work, we propose a unified federated learning (FL) framework that reconciles both HFL and VFL. Under this unified FL framework, we provide the No-Free-Lunch (NFL) theorem that quantifies the trade-off between privacy leakage, utility loss, and efficiency reduction for well-defined scenarios. We then leverage our proposed NFL theorem to analyze the lower bounds of the widely-adopted protection mechanisms, including the randomization mechanism, Paillier mechanism, secret sharing mechanism, and compression mechanism. 

Lots of problems are worth investigating. In this work, we use the Jensen-Shannon (JS) divergence instead of the commonly-used KL divergence to measure privacy leakage. The JS divergence is expressed as the summation of two KL divergences. Thus, it inherits the advantages of KL divergence. Besides, JS divergence satisfies the triangle inequality, which facilitates the theoretical analysis. One potential question is, if we use a generalized version of JS-divergence weighted using a hyperparameter $\alpha$ \cite{deasy2020constraining}, are we still able to bound the privacy leakage and quantify the trade-off analysis? 

The optimal privacy-utility-efficiency trade-off is cast as a constrained optimization problem in which the \textit{utility loss} and \textit{efficiency reduction} are minimized subject to a predefined constraint for \textit{privacy leakage}.  The optimization problem and the derived lower bounds provide an avenue for proposing meta-algorithms that search the optimal hyperparameter characterizing the protection mechanism. Designing a meta-algorithm that seeks to determine the ideal protection hyperparameter at each communication round is an intriguing problem. Whether it is possible to design an algorithm that can learn the hyperparameter adaptively is another promising research direction.

% Lots of promising problems are worth investigating. The privacy leakage is measured using the divergence between the prior and posterior beliefs. To measure the distance between the posterior distribution and the prior distribution, we use the Jensen-Shannon divergence instead of the commonly-used KL divergence. Jensen-Shannon divergence is expressed as the summation of two KL divergences. Therefore, it possesses the advantages of KL divergence. Besides, it satisfies the triangle inequality, which facilitates the theoretical analysis. One potential question is, if we use a generalized version of JS-divergence weighted using a hyperparameter $\alpha$ \cite{deasy2020constraining}, could we still provide bounds for privacy leakage and provide quantitative trade-off analyses? 

%% file: App_A_1_Quantitative_Relationship_Utility.tex
%\textbf{\Huge{Appendix}}

% \section{Quantifying $\epsilon_u, \epsilon_p,$ and $\epsilon_e$ using total variation distance}
%\section{Quantifying privacy leakage, utility loss and efficiency reduction using total variation distance}

\section{Outline of our work}\label{sec: outline_our_work}
\begin{figure}[h]
\centering
\includegraphics[width = 0.98\columnwidth]{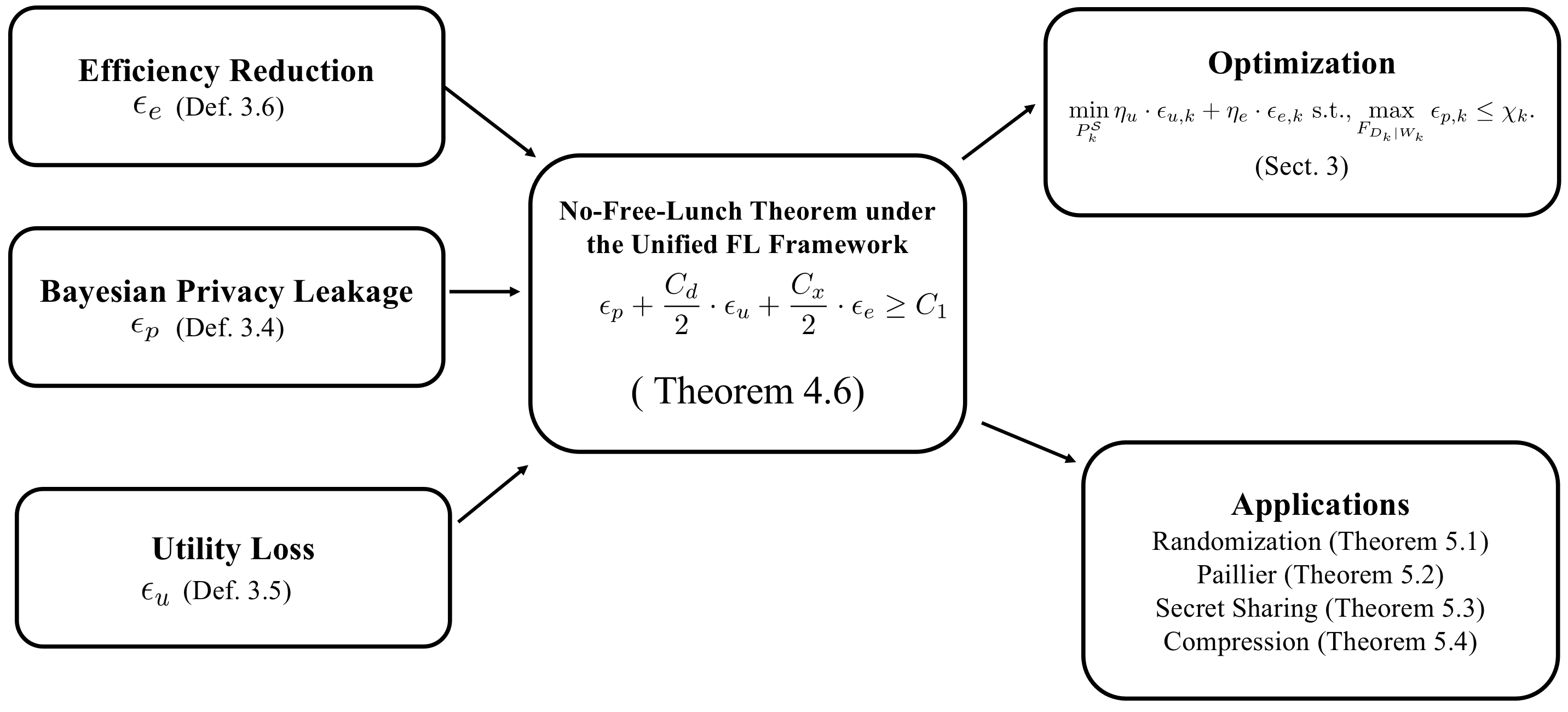}
\caption{Outline of our work.} \label{fig: our_framework}
\end{figure}

\section{Lower Bounds for Privacy Leakage, Utility Loss and Efficiency Reduction}

% \begin{lem}\label{lem: total_variation-utility trade-off_app_1}
% Let \pref{assump: assump_of_Delta} hold, and $\epsilon_{u}$ be defined in \pref{defi: utility_loss}. Let $P_{\text{fed}}^{\calRO}$ and $P_{\text{fed}}^{\calD}$ represent the distribution of the aggregated parameter before and after being protected. Then we have,
% \begin{align*}
%     \epsilon_{u} \ge\frac{\Delta}{2}\cdot {\text{TV}}(P^{\calRO}_{\text{fed}} || P^{\calD}_{\text{fed}} ).
% \end{align*}
% \end{lem}

%% file: App_A_2_Quantitative_Relationship_Privacy.tex
%\subsection{Quantitative relationship between total variation distance and $\epsilon_p$}
%\cite{zhang2022no} illustrated that 
When the private information $D$ is discrete, the privacy leakage could be lower bounded by the total variation distance between $P_k^{\calRO}$ and $P^{\calD}_k$, as is shown in the following lemma.

\begin{lem}\label{lem: total_variation-privacy trade-off_app_1}
Let $\epsilon_{p,k}$ be defined in \pref{defi: average_privacy_JSD}. Let $P_k^{\calRO}$ and $P^{\calD}_k$ represent the distribution of the parameter of client $k$ before and after being protected. Then for any client $k$, we have that

\begin{align}
    \epsilon_{p,k}\ge\sqrt{{\text{JS}}(F^{\calRO}_k || F^{\calO}_k)} - \frac{1}{2}(e^{2\xi}-1){\text{TV}}(P_k^{\calRO} || P^{\calD}_k).
\end{align}
Furthermore, we have that

\begin{align*}
\epsilon_p \ge \frac{1}{K}\sum_{k=1}^K \sqrt{{\text{JS}}(F^{\calRO}_k || F^{\calO}_k)} - \frac{1}{K}\sum_{k=1}^K \frac{1}{2}(e^{2\xi}-1)\cdot {\text{TV}}(P_k^{\calRO} || P^{\calD}_k).
\end{align*}
\end{lem}

To provide a lower bound for privacy leakage when the private information $D$ is discrete, we first show the following lemma. 

%% file: App_A_2_1_Lower_Bound_for_Privacy_Leakage_Discrete.tex
%\section{Lower Bound for Privacy Leakage}

\begin{lem}\label{lem: JSBound_app}
Let $P_k^{\calRO}$ and $P^{\calD}_k$ represent the distribution of the parameter of client $k$ before and after being protected. Let $F^{\calA}_k$ and $F^{\calRO}_k$ represent the belief of client $k$ about $S$ after observing the protected and original parameter. Then we have
\begin{align*}
{\text{JS}}(F^{\calA}_k || F^{\calRO}_k)\le \frac{1}{4}(e^{2\xi}-1)^2{\text{TV}}(P_k^{\calRO} || P^{\calD}_k)^2. 
\end{align*}
\end{lem}

\begin{proof}

Recall that 
${\text{JS}}(F^{\calA}_k || F^{\calO}_k) =
    \frac{1}{2}\int_{\mathcal{D}_k} f^{\calA}_{D_k}(d)\log\frac{f^{\calA}_{D_k}(d)}{f^{\calM}_{D_k}(d)}\textbf{d}\mu(d) + \frac{1}{2}\int_{\mathcal{D}_k} f^{\calO}_{D_k}(d)\log\frac{f^{\calO}_{D_k}(d)}{f^{\calM}_{D_k}(d)}\textbf{d}\mu(d) = \frac{1}{2}\int_{\mathcal{D}_k} \int_{\mathcal{W}_k^{\calD}} f_{{D_k}|{W_k}}(d|w)dP^{\calD}_{k}(w)\log\frac{\int_{\mathcal{W}_k^{\calD}} f_{{D_k}|{W_k}}(d|w)dP^{\calD}_{k}(w)}{f^{\calM}_{D_k}(d)}\textbf{d}\mu(d) + \frac{1}{2}\int_{\mathcal{D}_k} f_{D_k}(d)\log\frac{f_{D_k}(d)}{f^{\calM}_{D_k}(d)}\textbf{d}\mu(d)$. Denote $F^{\calM}_k = \frac{1}{2}(F^{\calA}_k+ F^{\calRO}_k)$, we have

\begin{align*}
{\text{JS}}(F^{\calA}_k || F^{\calRO}_k) & = \frac{1}{2}\left[KL\left(F^{\calA}_k, F^{\calM}_k\right) + KL\left(F^{\calRO}_k,F^{\calM}_k\right)\right]\\
& = \frac{1}{2}\left[\sum_{d\in\mathcal{D}_k} f^{\calA}_{D_k}(d)\log\frac{f^{\calA}_{D_k}(d)}{f^{\calM}_{D_k}(d)} + \sum_{d\in\mathcal{D}_k} f^{\calRO}_{D_k}(d)\log\frac{f^{\calRO}_{D_k}(d)}{f^{\calM}_{D_k}(d)}\right]\\
& = \frac{1}{2}\left[\sum_{d\in\mathcal{D}_k} f^{\calA}_{D_k}(d)\log\frac{f^{\calA}_{D_k}(d)}{f^{\calM}_{D_k}(d)} - \sum_{d\in\mathcal{D}_k} f^{\calRO}_{D_k}(d)\log\frac{f^{\calM}_{D_k}(d)}{f^{\calRO}_{D_k}(d)}\right]\\
&\le \frac{1}{2}\sum_{d\in\mathcal{D}_k}\left|f^{\calA}_{D_k}(d) - f^{\calRO}_{D_k}(d)\right|\left|\log\frac{f^{\calM}_{D_k}(d)}{f^{\calRO}_{D_k}(d)}\right|,
\end{align*}
where the inequality is due to $\frac{f^{\calA}_{D_k}(d)}{f^{\calM}_{D_k}(d)}\le \frac{f^{\calM}_{D_k}(d)}{f^{\calRO}_{D_k}(d)}$.\\
\textbf{Bounding $\left|f^{\calA}_{D_k}(d) - f^{\calRO}_{D_k}(d)\right|$.}
Let $\mathcal U_k = \{w\in\mathcal W_k: dP^{\calD}_k(w) - dP_k^{\calRO}(w)\ge 0\}$, and $\mathcal V_k = \{w\in\mathcal W_k: dP^{\calD}_k(w) - dP_k^{\calRO}(w)< 0\}$. Then we have 

\begin{align}\label{eq:initial_step_{JS}}
    \left|f^{\calA}_{D_k}(d) - f^{\calRO}_{D_k}(d)\right| &= \left|\int_{\mathcal W_k} f_{D_k|W_k}(d|w)[d P^{\calD}_k(w) - d P_k^{\calRO}(w)]\right|\nonumber\\
    &= \left|\int_{\mathcal{U}_k} f_{D_k|W_k}(d|w)[d P^{\calD}_k(w) - d P_k^{\calRO}(w)] + \int_{\mathcal{V}_k} f_{D_k|W_k}(d|w)[d P^{\calD}_k(w) - d P_k^{\calRO}(w)]\right|\nonumber\\
    %&\le\sup_{w\in\mathcal W^{\calD}_k \setminus \mathcal W^{\calRO}_k} f_{D_k|W_k}(d|w)\int_{\U} [d P^{\calD}_k(w) - d P_k^{\calRO}(w)] + \inf_{w\in\mathcal W^{\calRO}_k} f_{D_k|W_k}(d|w)\int_\V[d P^{\calD}_k(w) - d P_k^{\calRO}(w)]\nonumber\\
    &\le\left(\sup_{w\in\mathcal{W}_k} f_{D_k|W_k}(d|w) - \inf_{w\in\mathcal{W}_k} f_{D_k|W_k}(d|w)\right)\int_{\mathcal{U}_k} [d P^{\calD}_k(w) - d P_k^{\calRO}(w)].
\end{align}
Notice that

\begin{align*}
    \sup_{w\in\mathcal W_k} f_{D_k|W_k}(d|w) - \inf_{w\in\mathcal W_k} f_{D_k|W_k}(d|w) = \inf_{w\in\mathcal W_k} f_{D_k|W_k}(d|w)\left|\frac{\sup_{w\in\mathcal W_k} f_{D_k|W_k}(d|w)}{\inf_{w\in\mathcal W_k} f_{D_k|W_k}(d|w)}-1\right|.
\end{align*}
From the definition of $\xi$, we know that for any $w\in\mathcal W_k$,
\begin{align*}
    e^{-\xi}\le\frac{f_{D_k|W_k}(d|w)}{f_{D_k}(d)}\le e^{\xi},
\end{align*}
Therefore, for any pair of parameters $w,w'\in\mathcal W_k$, we have
\begin{align*}
    \frac{f_{D_k|W_k}(d|w)}{f_{D_k|W_k}(d|w')} = \frac{f_{D_k|W_k}(d|w)}{f_{D_k}(d)}/\frac{f_{D_k|W_k}(d|w')}{f_{D_k}(d)}\le e^{2\xi}. 
\end{align*}
Therefore, the first term of \pref{eq:initial_step_{JS}} is bounded by

\begin{align}\label{eq: bound_1_term_1_{JS}_ratio}
    \sup_{w\in\mathcal W_k} f_{D_k|W_k}(d|w) - \inf_{w\in\mathcal W_k} f_{D_k|W_k}(d|w) \le \inf_{w\in\mathcal W_k} f_{D_k|W_k}(d|w)(e^{2\xi}-1).
\end{align}
From the definition of total variation distance, we have
\begin{align}\label{eq: bound_1_term_2_{JS}_ratio}
    \int_{\U} [d P^{\calD}_k(w) - d P_k^{\calRO}(w)] = {\text{TV}}(P_k^{\calRO} || P^{\calD}_k).
\end{align}
Combining \pref{eq: bound_1_term_1_{JS}_ratio} and \pref{eq: bound_1_term_2_{JS}_ratio}, we have
\begin{align}\label{eq: bound_for_the_gap}
        |f^{\calA}_{D_k}(d) - f^{\calRO}_{D_k}(d)| &=\left(\sup_{w\in\mathcal W_k} f_{D_k|W_k}(d|w) - \inf_{w\in\mathcal W_k} f_{D_k|W_k}(d|w)\right)\int_\U [d P^{\calD}_k(w) - d P_k^{\calRO}(w)]\nonumber\\
        &\le\inf_{w\in\mathcal W_k} f_{D_k|W_k}(d|w)(e^{2\xi}-1){\text{TV}}(P_k^{\calRO} || P^{\calD}_k).
\end{align}

\textbf{Bounding $\left|\log\left(\frac{f^{\calM}_{D_k}(d)}{ f^{\calRO}_{D_k}(d)}\right)\right|.$} 
Since $\left|\log\left(\frac{a}{b}\right)\right| \le \frac{|a-b|}{\min\{a,b\}}$ for any two positive numbers $a$ and $b$, we have that

\begin{align}\label{eq: bound_for_log_ratio}
    \left|\log\frac{f^{\calM}_{D_k}(d)}{f^{\calRO}_{D_k}(d)}\right|&\le\frac{|f^{\calM}_{D_k}(d) - f^{\calRO}_{D_k}(d)|}{\min\{f^{\calM}_{D_k}(d), f^{\calRO}_{D_k}(d)\}}\nonumber\\
    &=\frac{|f^{\calA}_{D_k}(d) - f^{\calRO}_{D_k}(d)|}{2\min\{f^{\calM}_{D_k}(d), f^{\calRO}_{D_k}(d)\}}\nonumber\\
    &\le \frac{\inf_{w\in\mathcal W_k} f_{D_k|W_k}(d|w)(e^{2\xi}-1){\text{TV}}(P_k^{\calRO} || P^{\calD}_k)}{2\min\{f^{\calM}_{D_k}(d), f^{\calRO}_{D_k}(d)\}}\nonumber\\
    &\le \frac{1}{2}(e^{2\xi}-1){\text{TV}}(P_k^{\calRO} || P^{\calD}_k),
\end{align}
where the third inequality is due to $\min\{f^{\calM}_{D_k}(d), f^{\calRO}_{D_k}(d)\}\ge \min\{f^{\calA}_{D_k}(d), f^{\calRO}_{D_k}(d)\}\ge \inf\limits_{\small{w\in\mathcal W_k}} f_{D_k|W_k}(d|w)$.
Combining \pref{eq: bound_for_the_gap} and \pref{eq: bound_for_log_ratio}, we have
\begin{align*}
    {\text{JS}}(F^{\calA}_k || F^{\calRO}_k) & \le  \frac{1}{2}\left[\sum_{d\in\mathcal{D}_k} \left|(f^{\calA}_{D_k}(d) - f^{\calRO}_{D_k}(d))\right| \left|\log\frac{f^{\calM}_{D_k}(d)}{f^{\calRO}_{D_k}(d)}\right|\right]\\
    &\le \frac{1}{4}(e^{2\xi}-1)^2{\text{TV}}(P_k^{\calRO} || P^{\calD}_k)^2\sum_{d\in\mathcal{D}_k} \inf_{w\in\mathcal W^{\calRO}_k} f_{D_k|W_k}(d|w)\\
    &\le\frac{1}{4}(e^{2\xi}-1)^2{\text{TV}}(P_k^{\calRO} || P^{\calD}_k)^2.
\end{align*}

\end{proof}

With the above lemma, now we are ready to show \pref{lem: total_variation-privacy trade-off_app_1} when the private information $D$ is discrete, the analysis is similar to the continuous scenario. 
% \begin{lem}
% We have that 
% \begin{align*}
%     \sqrt{{\text{JS}}(F^{\calRO}_k || F^{\calO}_k)}\le\epsilon_p + \frac{1}{2}(e^{2\xi}-1){\text{TV}}(P_k^{\calRO} || P^{\calD}_k).
% \end{align*} 
% \end{lem}
\begin{proof}
The square root of JS divergence satisfies triangle inequality, which implies that
%Recall the triangle inequality shows that 
\begin{align*}
    \sqrt{{\text{JS}}(F^{\calRO}_k || F^{\calO}_k)} - \sqrt{{\text{JS}}(F^{\calA}_k || F^{\calO}_k)}\le \sqrt{{\text{JS}}(F^{\calA}_k || F^{\calRO}_k)},
\end{align*}
where $\sqrt{{\text{JS}}(F^{\calRO}_k || F^{\calO}_k)}$ is a problem-dependent constant, $\sqrt{{\text{JS}}(F^{\calA}_k || F^{\calO}_k)}$ represents the privacy leakage, and $\sqrt{{\text{JS}}(F^{\calA}_k || F^{\calRO}_k)}\le \frac{1}{2}(e^{2\xi}-1){\text{TV}}(P_k^{\calRO} || P^{\calD}_k)$ from \pref{lem: JSBound_app}. Therefore, $\forall k\in [K]$, we have that
\begin{align*}
    \sqrt{{\text{JS}}(F^{\calRO}_k || F^{\calO}_k)}&\le\epsilon_{p,k} + \frac{1}{2}(e^{2\xi}-1){\text{TV}}(P_k^{\calRO} || P^{\calD}_k).
\end{align*}

% For the encryption scheme, given any parameter $w$, the information gain of the attacker is close to zero due to semantic security. In fact, the information gain of the attacker is related to the guessing range of the attacker. It implies that the privacy leakage decreases with $n$.
\end{proof}

%% file: App_A_3_Quantitative_Relationship_Efficiency.tex
%\textbf{\Huge{Appendix}}

%\section{The quantitative relationship between ${\text{TV}}(P^{\calRO}_{k} || P^{\calD}_{k} )$ and $\epsilon_e$}

%\subsection{The quantitative relationship between ${\text{TV}}(P^{\calRO}_{k} || P^{\calD}_{k} )$ and $\epsilon_e$}

\begin{lem}\label{lem: total_variation-efficiency trade-off}
Let \pref{assump: assump_of_Xi_efficiency} hold, and $\epsilon_{e}$ be defined in \pref{defi: utility_loss}. Let $P_{k}^{\calRO}$ and $P_{k}^{\calD}$ represent the distribution of the federated model information before and after being protected. Let $\Xi = \min_{k\in[K]} \Xi_k$, and $\Gamma = \min_{k\in [K]}\Gamma_k$. Then we have,
\begin{align}
    \epsilon_{e} \ge\Xi\cdot\Gamma\cdot \frac{1}{K}\sum_{k = 1}^K {\text{TV}}(P^{\calRO}_{k} || P^{\calD}_{k} ).
\end{align}
\end{lem}

\begin{proof}

Let $\mathcal U_{k} = \{w\in\mathcal W_{k}: dP_{k}^{\calD}(w) - dP_{k}^{\calRO}(w)\ge 0\}$, and $\mathcal V_{k} = \{w\in\mathcal W_{k}: dP_{k}^{\calD}(w) - dP_{k}^{\calRO}(w)< 0\}$, where $\mathcal W_{k}$ represents the union of the supports of $P_{k}^{\calD}$ and $P_{k}^{\calRO}$. 

For any $w\in\mathcal V_{k}$, the definition of $\mathcal V_{k}$ implies that $dP_{k}^{\calRO}(w) > dP_{k}^{\calD}(w)\ge 0$. Therefore, $w$ belongs to the support of $P_{k}^{\calRO}$, which is denoted as $\mathcal W^{\calRO}_{k}$. Therefore we have that
\begin{align}\label{eq: subset_relationship_1_efficiency}
    \mathcal V_{k}\subset\mathcal W^{\calRO}_{k}.
\end{align}
Similarly, we have that
\begin{align}\label{eq: subset_relationship_n_efficiency}
    \mathcal U_{k}\subset\mathcal W^{\calD}_{k}.
\end{align}
Recall that $P_{k}^{\calD}$ represents the distribution of the aggregated parameter after being protected, and $p^{\calD}_{W_{k}}(w)$ represents the corresponding probability density function. 

$\calW_{k,\Xi_k}^{+}$ is defined as:
\begin{align*}
&\calW_{k,\Xi_k}^{+} = \left\{w\in\mathcal U_{k}: C(w^{\calRO}_{\text{max}}) + 2 \cdot \Xi_k \le C(w) \quad\text{and}\quad p_{W_{k}}^{\calD}(w) \ge 2 \cdot p^{\calRO}_{W_{k}}(w)\right\},
\end{align*}
where $w^{\calRO}_{\text{max}} = \arg\max_{w\in\calW_k^{\calRO}} C(w)$.

% and the probability density of these distorted parameters are larger than that of the original parameters

We denote $\calW_{k,\Xi_k}^{-} = \mathcal U_{k}\setminus\calW_{k,\Xi_k}^{+}$. From \pref{assump: assump_of_Xi_efficiency}, $\Gamma_k$ represents the maximum constant satisfying that
\begin{align}
    \int_{\mathcal U_{k}}(\one\{w\in\calW_{k,\Xi_k}^{+}\} - \one\{w\in\calW_{k,\Xi_k}^{-}\}) p^{\calD}_{W_{k}}(w) dw\ge\Gamma_k\cdot{\text{TV}}(P_{k}^{\calRO} || P_{k}^{\calD} ),
\end{align}
where $p^{\calD}_{W_{k}}$ denotes the probability density function of the protected model information and $p^{\calRO}_{W_{k}}$ denotes the probability density function of the original model information. Then we have

\begin{align*}
\epsilon_{e,k} &=\left[\mathbb E_{w\sim P_{k}^{\calD}}[C(w)] - \mathbb E_{w\sim P_{k}^{\calRO}}[C(w)]\right]\\
     &=\left[\int_{\mathcal W_{k}} C(w) dP^{\calD}_{k}(w) - \int_{\mathcal W_{k}} C(w)dP^{\calRO}_{k}(w)\right]\\
     &=\left[\int_{\mathcal U_{k}} C(w)[d P_{k}^{\calD}(w) - d P_{k}^{\calRO}(w)] - \int_{\mathcal{V}_{k}} C(w)[d P_{k}^{\calRO}(w) - d P_{k}^{\calD}(w)]\right]\\
% &= \int_{\mathcal U_{k}} C(w)\one\{w\in\calW_{k,\Xi_k}^{+}\}[d P_{k}^{\calD}(w) - d P_{k}^{\calRO}(w)]  + \int_{\mathcal U_{k}} C(w)\one\{w\in\calW_{k,\Xi_k}^{-}\}[d P_{k}^{\calD}(w) - d P_{k}^{\calRO}(w)] \\
% & + \int_{\mathcal U_{k}} C(w)\one\{w\not\in\{\calW_{k,\Xi_k}^{+}\cup\calW_{k,\Xi_k}^{-}\}\}[d P_{k}^{\calD}(w) - d P_{k}^{\calRO}(w)] - \int_{\mathcal V_{k}} C(w)\one\{w\in\mathcal W^{\calRO}_{k}\}[d P_{k}^{\calRO}(w) - d P_{k}^{\calD}(w)]\\
&= \int_{\mathcal U_{k}} C(w)\one\{w\in\calW_{k,\Xi_k}^{+}\}[d P_{k}^{\calD}(w) - d P_{k}^{\calRO}(w)] + \int_{\mathcal U_{k}} C(w)\one\{w\in\calW_{k,\Xi_k}^{-}\}[d P_{k}^{\calD}(w) - d P_{k}^{\calRO}(w)]\\
& - \int_{\mathcal V_{k}} C(w)\one\{w\in\mathcal W^{\calRO}_{k}\}[d P_{k}^{\calRO}(w) - d P_{k}^{\calD}(w)]\\
&\ge 2\Xi_k\cdot\int_{\mathcal U_{k}}\one\{w\in\calW_{k,\Xi_k}^{+}\} (p^{\calD}_{W_{k}}(w) - p^{\calRO}_{W_{k}}(w)) dw - \Xi_k\cdot\int_{\mathcal U_{k}}\one\{w\in\calW_{k,\Xi_k}^{-}\} (p^{\calD}_{W_{k}}(w) - p^{\calRO}_{W_{k}}(w)) dw\\
&\ge 2\Xi_k\cdot\int_{\mathcal U_{k}}\one\{w\in\calW_{k,\Xi_k}^{+}\} \frac{p^{\calD}_{W_{k}}(w)}{2} dw - \Xi_k\cdot\int_{\mathcal U_{k}}\one\{w\in\calW_{k,\Xi_k}^{-}\} p^{\calD}_{W_{k}}(w) dw\\
&\ge\Xi_k\cdot\Gamma_k\cdot {\text{TV}}(P^{\calRO}_{k} || P^{\calD}_{k} ),
\end{align*}
where the first inequality is due to $C(w^{\calRO}_{\text{max}}) + 2\Xi_k \le C(w)$, $\forall w\in\calW_{k,\Xi_k}^{+}$, and $C(w^{\calRO}_{\text{max}}) - \Xi_k \le C(w)$, $\forall w\in\calU_k$, the second inequality is due to the  $p_{W_{k}}^{\calD}(w) - p^{\calRO}_{W_{k}}(w)\ge \frac{p_{W_{k}}^{\calD}(w)}{2}$, $\forall w\in\calW_{k,\Xi_k}^{+}$, and the third inequality is due to $\int_{\mathcal U_{k}}(\one\{w\in\calW_{k,\Xi_k}^{+}\} - \one\{w\in\calW_{k,\Xi_k}^{-}\}) p^{\calD}_{W_{k}}(w) dw\ge\Gamma_k\cdot{\text{TV}}(P_{k}^{\calRO} || P_{k}^{\calD} )$ from \pref{assump: assump_of_Xi_efficiency}. 
%and $\mathcal V_{k}\subset\mathcal W^{\calRO}_{k}$ and $\mathcal U_{k}\subset\mathcal W^{\calD}_{k}$.

Therefore, we have that
\begin{align*}
    \epsilon_e = \frac{1}{K}\sum_{k = 1}^K \epsilon_{e,k}
    &\ge \frac{1}{K}\sum_{k = 1}^K \Xi_k\cdot\Gamma_k\cdot{\text{TV}}(P^{\calRO}_{k} || P^{\calD}_{k} )\\
    &\ge\Xi\cdot\Gamma\cdot \frac{1}{K}\sum_{k = 1}^K {\text{TV}}(P^{\calRO}_{k} || P^{\calD}_{k} ).
\end{align*}

\end{proof}

%% file: App_B_Trade_offs.tex
%\section{Privacy-Utility-Efficiency Trade-off}

\section{Trade-off Between Privacy, Utility and Efficiency}\label{sec: trade_off_privacy_utility_efficiency}

First, we quantify the trade-off between privacy and efficiency. The following lemma illustrates that the summation of the privacy leakage and the efficiency reduction is lower bounded by a problem-dependent constant.

\begin{lem}[No free lunch theorem (NFL) for privacy and efficiency]\label{thm: efficiency-privacy trade-off_JSD_app}
Let $\epsilon_p$ be defined in Def. \ref{defi: average_privacy_JSD}, and let $\epsilon_e$ be defined in Def. \ref{defi: efficiency_reduction}, with \pref{assump: assump_of_Xi_efficiency} we have:
\begin{align*}
    \epsilon_{p} + C_x\cdot\epsilon_e\ge C_1,
\end{align*}
where $C_1 = \frac{1}{K}\sum_{k=1}^K \sqrt{{\text{JS}}(F^{\calRO}_k || F^{\calO}_k)}$ and $C_x = \frac{1}{2\Xi\Gamma}(e^{2\xi}-1)$.
\end{lem}

\begin{proof}

First, from \pref{lem: total_variation-privacy trade-off} we have that 

\begin{align*}
    \epsilon_p\ge \frac{1}{K}\sum_{k=1}^K \sqrt{{\text{JS}}(F^{\calRO}_k || F^{\calO}_k)} - \frac{1}{K}\sum_{k=1}^K \frac{1}{2}(e^{2\xi}-1)\cdot{\text{TV}}(P_k^{\calRO} || P^{\calD}_k),
\end{align*}
From \pref{lem: efficiency_reduction_and_tvd}, we have
\begin{align*}
    \epsilon_e\ge\Xi\cdot\Gamma\cdot \frac{1}{K}\sum_{k = 1}^K {\text{TV}}(P^{\calRO}_{k} || P^{\calD}_{k} ).
\end{align*}
Combining the above two equations, we have that
\begin{align*}
    \frac{1}{K}\sum_{k=1}^K \sqrt{{\text{JS}}(F^{\calRO}_k || F^{\calO}_k)} \le\epsilon_{p} + \frac{1}{2\Xi\Gamma}(e^{2\xi}-1)\epsilon_e,
\end{align*}
where $\epsilon_{p} = \frac{1}{K}\sum_{k=1}^K \epsilon_{p,k}$.\\
The above equation could be further simplified as
\begin{align*}
    C_1 &\le\epsilon_{p} + C_x\epsilon_e,
\end{align*}
where $C_1 = \frac{1}{K}\sum_{k=1}^K \sqrt{{\text{JS}}(F^{\calRO}_k || F^{\calO}_k)}$ and $C_x = \frac{1}{2\Xi\Gamma}(e^{2\xi}-1)$.
\end{proof}

%\subsection{Trade-off Between Privacy and Utility}
The trade-off between privacy and utility was the main result shown in \cite{zhang2022no}, the analysis of which is applicable to a more general definition for utility loss (\pref{defi: utility_loss}). 
The following lemma illustrates that utility loss is lower bounded by the total variation distance between the unprotected and protected distributions. 
\begin{lem}\cite{zhang2022no}\label{lem: total_variation-utility trade-off_app}
Let \pref{assump: assump_of_Delta} hold, and $\epsilon_{u}$ be defined in \pref{defi: utility_loss}. Let $P_{\text{fed}}^{\calRO}$ and $P_{\text{fed}}^{\calD}$ represent the distribution of the aggregated parameter before and after being protected. Then we have,
\begin{align*}
    \epsilon_{u} \ge\frac{\Delta}{2}\cdot {\text{TV}}(P^{\calRO}_{\text{fed}} || P^{\calD}_{\text{fed}} ).
\end{align*}
\end{lem}
With \pref{lem: total_variation-privacy trade-off_app_1} and \pref{lem: total_variation-utility trade-off_app}, it is now natural to provide a quantitative relationship between the utility loss and the privacy leakage (\pref{thm: utility-privacy trade-off_JSD_app}).

\begin{thm}\cite{zhang2022no}[No free lunch theorem (NFL) for privacy and utility]\label{thm: utility-privacy trade-off_JSD_app} 
Let $\epsilon_p$ be defined in Def. \ref{defi: average_privacy_JSD}, and let $\epsilon_u$ be defined in Def. \ref{defi: utility_loss} at the convergence step, with \pref{assump: assump_of_Delta} we have:
\begin{align}\label{eq: total_variation-privacy trade-off_app_02}
 \epsilon_{p} + C_d\cdot \epsilon_{u}\ge C_1,
\end{align}
where $\xi_k$=$\max_{w\in \mathcal{W}_k, d \in \mathcal{D}_k} \left|\log\left(\frac{f_{D_k|W_k}(d|w)}{f_{D_k}(d)}\right)\right|$, and $\xi$=$\max_{k\in [K]} \xi_k$ represents the maximum privacy leakage over all possible information $w$ released by client $k$, and $[K] = \{1,2,\cdots, K\}$, $C_1 = \frac{1}{K}\sum_{k=1}^K \sqrt{{\text{JS}}(F^{\calRO}_k || F^{\calO}_k)}$, $C_d = \frac{\gamma}{4\Delta}(e^{2\xi}-1)$, and $\gamma = \frac{\frac{1}{K}\sum_{k=1}^K {\text{TV}}(P_k^{\calRO} || P^{\calD}_k)}{{\text{TV}}(P^{\calRO}_{\text{fed}} || P^{\calD}_{\text{fed}} )}$.
\end{thm}

% \begin{proof}
% From \pref{lem: total_variation-privacy trade-off}, we have

% \begin{align}\label{eq: converge_eq_1_modify_1_0}
%     \frac{1}{K}\sum_{k=1}^K \sqrt{{\text{JS}}(F^{\calRO}_k || F^{\calO}_k)} \le\frac{1}{K}\sum_{k=1}^K \epsilon_{p,k} + \frac{1}{K}\sum_{k=1}^K \frac{1}{2}(e^{2\xi}-1){\text{TV}}(P_k^{\calRO} || P^{\calD}_k).
% \end{align}

% From \pref{lem: total_variation-utility trade-off_app}, we have
% \begin{align}\label{eq: converge_eq_2_0_app_01}
%     \epsilon_{u} \ge \frac{\Delta}{2}\cdot {\text{TV}}(P^{\calRO}_{\text{fed}} || P^{\calD}_{\text{fed}} ).
% \end{align}

% Combining \pref{eq: converge_eq_1_modify_1_0} and \pref{eq: converge_eq_2_0_app_01}, we have that

% \begin{align*}
%     \frac{1}{K}\sum_{k=1}^K \sqrt{{\text{JS}}(F^{\calRO}_k || F^{\calO}_k)} \le\epsilon_{p} + \frac{\gamma}{4\Delta}(e^{2\xi}-1)\epsilon_u,
% \end{align*}
% where $\gamma = \frac{\frac{1}{K}\sum_{k=1}^K {\text{TV}}(P_k^{\calRO} || P^{\calD}_k)}{{\text{TV}}(P^{\calRO}_{\text{fed}} || P^{\calD}_{\text{fed}} )}$, and $\epsilon_{p} = \frac{1}{K}\sum_{k=1}^K \epsilon_{p,k}$.

% The above equation could be further simplified as
% \begin{align*}
%     C_1 &\le\epsilon_{p} + C_{d}\epsilon_u,
% \end{align*}
% where $C_1 = \frac{1}{K}\sum_{k=1}^K \sqrt{{\text{JS}}(F^{\calRO}_k || F^{\calO}_k)}$ and $C_{d} = \frac{\gamma}{4\Delta}(e^{2\xi}-1)$.
% \end{proof}

With \pref{thm: efficiency-privacy trade-off_JSD_app} and \pref{thm: utility-privacy trade-off_JSD_app}, our main result \pref{thm: utility-privacy-efficiency Trade-off_JSD_mt} is proven.  

%% file: App_C_1_Randomization.tex
%\section{Applications}

\section{Analysis for Randomization Mechanism} \label{sec: analysis_for_Randomization}

\textit{Randomization} mechanism adds random noise such as Gaussian noise to model gradients \cite{abadi2016deep,geyer2017differentially,truex2020ldp}.

\begin{itemize}
    \item Let $W_k^{\calRO}$ be the model information sampled from distribution $P_k^{\calRO} = \calN(\mu_0,\Sigma_0)$, where $\mu_0 \in \mathbb{R}^n$, $\Sigma_0 = \text{diag}(\sigma_{1}^2,\cdots, \sigma_{m}^2)$ is a diagonal matrix.
    \item The protected model information $W_k^{\calD} = W_k^{\calRO} + \epsilon_k$, where $\epsilon_k \sim \calN(0, \Sigma_\epsilon)$ and $\Sigma_\epsilon = \text{diag}(\sigma_\epsilon^2, \cdots, \sigma_\epsilon^2)$. Therefore, $W_k^{\calD}$ follows the distribution $P_k^{\calD} = \calN(\mu_0, \Sigma_0+ \Sigma_\epsilon)$.
    \item The protected model information $W_{\text{fed}}^{\calD} = \frac{1}{K}\sum_{k=1}^K (W_k^{\calRO} + \epsilon_k)$ follows distribution $P^{\calD}_{\text{fed}} = \calN(\mu_0, \Sigma_0/K+ \Sigma_\epsilon/K)$.
\end{itemize}

%\subsection{Privacy-Utility-Efficiency Trade-off for Randomization Mechanism}

The following lemmas establish bounds for the privacy leakage and utility loss, and efficiency reduction using the variance of the noise $\sigma_\epsilon^2$.

\begin{lem}\label{lem: epsilon_p_randomization_mechanism}
For randomization mechanism, the privacy leakage is bounded by
\begin{align*}
    \epsilon_p\ge C_1 - \frac{3C_2}{2}\cdot\min\left\{1, \sigma_\epsilon^2\sqrt{\sum_{i=1}^{m}\frac{1}{\sigma_i^4}} \right\}, 
\end{align*}
and the efficiency reduction is  
\begin{align*}
    \epsilon_e\ge\frac{\Xi\Gamma}{100}\min\left\{1, \sigma_\epsilon^2\sqrt{\sum_{i=1}^{m}\frac{1}{\sigma_i^4}} \right\}. 
\end{align*}
\end{lem}

\begin{proof}
From \pref{lem: total_variation-privacy trade-off} we have that

\begin{align*}
\epsilon_{p}\ge\frac{1}{K}\sum_{k=1}^K \sqrt{{\text{JS}}(F^{\calRO}_k || F^{\calO}_k)} - \frac{1}{K}\sum_{k=1}^K \frac{1}{2}(e^{2\xi}-1)\cdot {\text{TV}}(P_k^{\calRO} || P^{\calD}_k).
\end{align*}
From Lemma C.2 of \cite{zhang2022no}, we have that
\begin{align}\label{eq: lower_bound_and_upper_bound}
    \frac{1}{100}\min\left\{1, \sigma_\epsilon^2\sqrt{\sum_{i=1}^{m}\frac{1}{\sigma_i^4}} \right\}\le{\text{TV}}(P_k^{\calRO} || P^{\calD}_k )\le\frac{3}{2}\min\left\{1, \sigma_\epsilon^2\sqrt{\sum_{i=1}^{m}\frac{1}{\sigma_i^4}} \right\}.
\end{align}
Then we have that
\begin{align*}
    \epsilon_p\ge C_1 - \frac{3C_2}{2}\cdot\min\left\{1, \sigma_\epsilon^2\sqrt{\sum_{i=1}^{m}\frac{1}{\sigma_i^4}} \right\}. 
\end{align*}
From \pref{lem: efficiency_reduction_and_tvd} and \pref{eq: lower_bound_and_upper_bound}, we have that

\begin{align*}
    \epsilon_{e} &\ge\frac{1}{K}\sum_{k=1}^K\Xi\cdot\Gamma\cdot {\text{TV}}(P^{\calRO}_{k} || P^{\calD}_{k} )\\
    &\ge\frac{\Xi\Gamma}{100}\min\left\{1, \sigma_\epsilon^2\sqrt{\sum_{i=1}^{m}\frac{1}{\sigma_i^4}} \right\}.
\end{align*}

\end{proof}

\begin{lem}\label{lem: epsilon_u_randomization_mechanism}
For randomization mechanism, the utility loss is bounded by
\begin{align*}
    \epsilon_u\ge\frac{\Delta}{200}\cdot\min\left\{1, \sigma_\epsilon^2\sqrt{\sum_{i=1}^{m}\frac{1}{\sigma_i^4}} \right\}. 
\end{align*}
\end{lem}
\begin{proof}
From Lemma C.2 of \cite{zhang2022no}, we have that
\begin{align}\label{eq: D_2_1}
    {\text{TV}}(P^{\calRO}_{\text{fed}} || P^{\calD}_{\text{fed}} )\ge \frac{1}{100}\min\left\{1, \sigma_\epsilon^2\sqrt{\sum_{i=1}^{m}\frac{1}{\sigma_i^4}} \right\}.
\end{align}
From \pref{lem: total_variation-utility trade-off_mt}, we have that
\begin{align}\label{eq: D_2_2}
    \epsilon_{u} \ge\frac{\Delta}{2}\cdot {\text{TV}}(P^{\calRO}_{\text{fed}} || P^{\calD}_{\text{fed}} ).
\end{align}
Combining the \pref{eq: D_2_1} and \pref{eq: D_2_2}, we have
\begin{align*}
    \epsilon_u\ge\frac{\Delta}{200}\cdot\min\left\{1, \sigma_\epsilon^2\sqrt{\sum_{i=1}^{m}\frac{1}{\sigma_i^4}} \right\}. 
\end{align*}

\end{proof}

%% file: App_C_2_Paillier.tex
\section{Analysis for Paillier Homomorphic Encryption}\label{sec: analysis_for_Paillier}

% The following theorem provides an estimation for $\xi$. 
% \begin{thm}
% From the exposed ciphertext, we know that $q\in [2^c, 2^c\cdot\beta].$
% \end{thm}

% \cite{hardy2017private}.

The \textbf{Paillier} encryption mechanism was proposed by \cite{paillier1999public} is an asymmetric additive homomorphic encryption mechanism, which was widely applied in FL \cite{zhang2019pefl, aono2017privacy, truex2019hybrid, cheng2021secureboost}. We first introduce the basic definition of Paillier algorithm in federated learning \cite{fang2021privacy}. Paillier encyption contains three parts including key generation, encryption and decryption. Let $h$ represent the plaintext, and $c$ represent the ciphertext.

\paragraph{Key Generation} Let ($n,g$) represent the public key, and ($\lambda, \mu$) represent the private key. Select two primes $p$ and $q$ that are rather large, satisfying that $\text{gcd}(pq, (p-q)(q-1)) = 1$. Select $g$ randomly satisfying that $g\in \mathbb Z_{n^2}^*$. Let $n = p \cdot q$, $\lambda = \text{lcm} (p-1, q-1)$, and $\mu = (L(g^{\lambda}\text{ mod }n^2))^{-1}\text{ mod }n$.

\paragraph{Encryption} Randomly select $r$ and encode $h$ as:

\begin{align*}
    c = g^h\cdot r^n \text{ mod } n^2,
\end{align*}
where $n = p \cdot q$, $p$ and $q$ are two selected primes. Note that $g$ is an integer selected randomly, and $g\in \mathbb Z_{n^2}^*$. Therefore, $n$ can divide the order of $g$.\\ 

\paragraph{Decryption} Using the private key $(\lambda, \mu)$ to decrypt the ciphertext $c$ as:
\begin{align*}
    h = L(c^{\lambda}\text{ mod }n^2)\cdot \mu \text{ mod } n,
\end{align*}
where $L(x) = \frac{x-1}{n}$, $\mu = (L(g^{\lambda}\text{ mod }n^2))^{-1}\text{ mod }n$, and $\lambda = \text{lcm} (p-1, q-1)$.

Let $m$ represent the dimension of the parameter.
\begin{itemize}
    \item Let $W_k^{\calRO}$ represent the plaintext that follows a uniform distribution over $[a_k^1 - \delta, a_k^1 + \delta]\times [a_k^2 - \delta, a_k^2 + \delta]\times\cdots\times [a_k^{m} - \delta, a_k^{m} + \delta]$.
    \item Assume that the ciphertext $W_k^{\calD}$ follows a uniform distribution over $[0,n^2-1]^{m}$.
    \item Let $W_{\text{fed}}^{\calRO}$ represent the federated plaintext that follows a uniform distribution over $[\bar a^1 - \delta, \bar a^1 + \delta]\times [\bar a^2 - \delta, \bar a^2 + \delta]\cdots\times [\bar a^{m} - \delta, \bar a^{m} + \delta]$, where $\bar a_i = \sum_{k = 1}^K a_k^i$.
    \item The federated ciphertext $W_{\text{fed}}^{\calD}$ follows a uniform distribution over $[0,n^2-1]^{m}$.
\end{itemize}

Intuitively, longer ciphertext should theoretically increase efficiency and decrease privacy leakage. The following lemma provides lower bounds for privacy leakage and efficiency reduction for Paillier mechanism.
\begin{lem}\label{lem: epsilon_p_and_one_over_n_app_01}
For Paillier mechanism, the privacy leakage is bounded by
\begin{align}\label{eq privacy_bound_app_02}
    \epsilon_p \ge C_1 - C_2\cdot\left[ 1 - \left(\frac{2\delta}{n^2}\right)^m\right],
\end{align}
and the efficiency reduction is bounded by
\begin{align*}
    \epsilon_e \ge \Xi\cdot\Gamma\cdot\left[1 - \left(\frac{2\delta}{n^2}\right)^{m}\right]. 
\end{align*}
\end{lem}

\begin{proof}
Let $W^{\calRO}_k$ represent the plaintext $m$, and $W^{\calD}_k$ represent the ciphertext $c$. Recall for encryption, we have that
\begin{align*}
    c = g^m\cdot r^n \text{ mod } n^2.
\end{align*}

The ciphertext $W_k^{\calD}$ follows a uniform distribution over $[0,n^2-1]^{m}$, and the plaintext $W_k^{\calRO}$ follows a uniform distribution over $[a_k^1 - \delta, a_k^1 + \delta]\times [a_k^2 - \delta, a_k^2 + \delta]\cdots\times [a_k^{m} - \delta, a_k^{m} + \delta]$, and $a_k^i\in [0,n^2-1]$, $\forall i = 1,2, \cdots, m$. Then we have that
\begin{align*}
    \text{TV}(P^{\calRO}_k || P^{\calD}_k) 
    &= \int_{[a_k^1 - \delta, a_k^1 + \delta]}\int_{[a_k^2 - \delta, a_k^2 + \delta]}\cdots\int_{[a_k^{m} - \delta, a_k^{m} + \delta]} \left(\left(\frac{1}{2\delta}\right)^{m} - \left(\frac{1}{n^2}\right)^{m}\right) dw_1 dw_2 \cdots d{w_{m}}\\
    & = \left[\left(\frac{1}{2\delta}\right)^{m} - \left(\frac{1}{n^2}\right)^{m}\right]\cdot (2\delta)^{m}.
\end{align*}

From \pref{lem: total_variation-privacy trade-off} we have that

\begin{align*}
\epsilon_{p}\ge\frac{1}{K}\sum_{k=1}^K \sqrt{{\text{JS}}(F^{\calRO}_k || F^{\calO}_k)} - \frac{1}{K}\sum_{k=1}^K \frac{1}{2}(e^{2\xi}-1)\cdot {\text{TV}}(P_k^{\calRO} || P^{\calD}_k).
\end{align*}

Combining the above two equations, we have that

\begin{align*}
    \epsilon_p &\ge \frac{1}{K}\sum_{k=1}^K \sqrt{{\text{JS}}(F^{\calRO}_k || F^{\calO}_k)} - \frac{1}{K}\sum_{k=1}^K \frac{1}{2}(e^{2\xi}-1)\cdot {\text{TV}}(P_k^{\calRO} || P^{\calD}_k)\nonumber\\ &=\sqrt{{\text{JS}}(F^{\calO}_k || F^{\calRO}_k)} - \frac{1}{K}\sum_{k=1}^K\frac{1}{2}(e^{2\xi}-1)\cdot \left[\left(\frac{1}{2\delta}\right)^{m} - \left(\frac{1}{n^2}\right)^{m}\right]\cdot (2\delta)^{m}\nonumber\\
    &= C_1 - \frac{1}{K}\sum_{k=1}^K\frac{1}{2}(e^{2\xi}-1)\cdot \left[\left(\frac{1}{2\delta}\right)^{m} - \left(\frac{1}{n^2}\right)^{m}\right]\cdot (2\delta)^{m}\nonumber\\
    & = C_1 - C_2\cdot\left[ 1 - \left(\frac{2\delta}{n^2}\right)^m\right].
\end{align*}

From \pref{lem: efficiency_reduction_and_tvd}, we have that
\begin{align*}
    \epsilon_{e} &\ge\frac{1}{K}\sum_{k=1}^K\Xi\cdot\Gamma\cdot {\text{TV}}(P^{\calRO}_{k} || P^{\calD}_{k} )\\
    &\ge\Xi\cdot\Gamma\cdot\left[1 - \left(\frac{2\delta}{n^2}\right)^{m}\right].
\end{align*}
\end{proof}

For Paillier mechanism, the distorted parameter given secret key becomes the original parameter. The following lemma shows that the utility loss for Paillier mechanism is $0$.
\begin{lem}
For Paillier mechanism, the utility loss $\epsilon_u = 0$.
\end{lem}
\begin{proof}
Let $P^{\calD}_{\text{fed}}$ represent the distribution of the distorted parameter which is decrypted by the client. 
    Note that ${\text{TV}}(P^{\calRO}_{\text{fed}} || P^{\calD}_{\text{fed}} ) = 0$. From Lemma C.3 of \cite{zhang2022no}, the utility loss is equal to $0$.
\end{proof}

%% file: App_C_3_Secret_Sharing.tex
\section{Analysis for Secret Sharing Mechanism}\label{sec: analysis_for_Secret_Sharing}
Many MPC-based protocols (particularly secret sharing) are used to build secure machine learning models, such as linear regression, logistic regression, recommend systems, and so on. \cite{SecShare-Adi79,SecShare-Blakley79,bonawitz2017practical} were developed to distribute a secret among a group of participants.

Let $m$ represent the dimension of the model information.
\begin{itemize}
    \item Let $W_k^{\calRO}$ represent the original model information that follows a uniform distribution over $[a_k^1 - \delta, a_k^1 + \delta]\times [a_k^2 - \delta, a_k^2 + \delta]\cdots\times [a_k^{m} - \delta, a_k^{m} + \delta]$.
    \item Let $W_k^{\calD}$ represent the distorted model information that follows a uniform distribution over $[a_k^1 - b_k^1, a_k^1 + r_k^1]\times [a_k^2 - b_k^2, a_k^2 + r_k^2]\cdots\times [a_k^{m} - b_k^{m}, a_k^{m} + r_k^{m}]$.
\end{itemize}

\begin{lem}
For secret sharing mechanism, the lower bound for privacy leakage is
\begin{align*}
    \epsilon_p \ge C_1 - C_2\cdot\frac{1}{K}\sum_{k=1}^K \left( 1 - \prod_{j = 1}^{m}\frac{2\delta}{b_k^j + r_k^j}\right).
\end{align*}
\end{lem}

\begin{proof}
%The ciphertext $c$ could be regarded as a random variable which is uniformly distributed over $[c-a,c+b]$. 
% Assume that $P_k^{\calRO}$ is close to a one-point distribution ($P_k^{\calRO} = \calU_{[c_k - \delta_k,c_k + \delta_k]}$), and $P^{\calD}_k$ is the uniform distribution over $[c_k-a_k,c_k+b_k]$ ($P^{\calD}_k = \calU_{[c_k-a_k, c_k+b_k]}$), and $a_k,b_k>\delta_k > 0$. Then, we have
% \begin{align*}
%     \text{TV}(P_k^{\calRO} || P^{\calD}_k) 
%     & = (\frac{1}{2\delta_k} - \frac{1}{b_k + a_k})\cdot 2\delta_k = 1 - \frac{2\delta_k}{b_k + a_k}.
% \end{align*}

Notice that $W_k^{\calD}$ follows a uniform distribution over $[a_k^1 - b_k^1, a_k^1 + r_k^1]\times [a_k^2 - b_k^2, a_k^2 + r_k^2]\cdots\times [a_k^{m} - b_k^{m}, a_k^{m} + r_k^{m}]$, and $W_k^{\calRO}$ follows a uniform distribution over $[a_k^1 - \delta, a_k^1 + \delta]\times [a_k^2 - \delta, a_k^2 + \delta]\cdots\times [a_k^{m} - \delta, a_k^{m} + \delta]$, and $\delta< b_k^{m}, r_k^{m}$, $\forall i = 1,2, \cdots, m$. Then we have that
\begin{align*}
    &\text{TV}(P^{\calRO}_k || P^{\calD}_k)\\ 
    &= \int_{[a_k^1 - \delta, a_k^1 + \delta]}\int_{[a_k^2 - \delta, a_k^2 + \delta]}\cdots\int_{[a_k^{m} - \delta, a_k^{m} + \delta]} \left(\left(\frac{1}{2\delta}\right)^{m} - \prod_{j = 1}^{m}\left(\frac{1}{b_k^j + r_k^j}\right)\right) dw_1 dw_2 \cdots d{w_{m}}\\
    & = \left(\left(\frac{1}{2\delta}\right)^{m} - \prod_{j = 1}^{m}\left(\frac{1}{b_k^j + r_k^j}\right)\right)\cdot (2\delta)^{m}.
\end{align*}
Therefore, we have that

\begin{align*}
\epsilon_p &\ge \frac{1}{K}\sum_{k=1}^K \sqrt{{\text{JS}}(F^{\calRO}_k || F^{\calO}_k)} - \frac{1}{K}\sum_{k=1}^K \frac{1}{2}(e^{2\xi}-1)\cdot {\text{TV}}(P_k^{\calRO} || P^{\calD}_k)\\
& = \frac{1}{K}\sum_{k=1}^K \sqrt{{\text{JS}}(F^{\calRO}_k || F^{\calO}_k)} - \frac{1}{K}\sum_{k=1}^K \frac{1}{2}(e^{2\xi}-1)\cdot\left(\left(\frac{1}{2\delta}\right)^{m} - \prod_{j = 1}^{m}\left(\frac{1}{b_k^j + r_k^j}\right)\right)\cdot (2\delta)^{m}.
\end{align*}
where the first inequality is due to \pref{eq: total_variation-privacy trade-off} in \pref{lem: total_variation-privacy trade-off}.
Then we have that
\begin{align*}
    \epsilon_p \ge C_1 - C_2\cdot\frac{1}{K}\sum_{k=1}^K \left( 1 - \prod_{j = 1}^{m}\frac{2\delta}{b_k^j + r_k^j}\right).
\end{align*}
\end{proof}

\begin{lem}
For secret sharing mechanism, the utility loss $\epsilon_u = 0$.
\end{lem}
\begin{proof}
    For secret sharing mechanism, the federated model information does not change after being protected, which implies that $P_{\text{fed}}^{\calRO} = P_{\text{fed}}^{\calD}$. Therefore, we have that 
\begin{align*}
&\epsilon_{u} =\frac{1}{K}\sum_{k=1}^K   \epsilon_{u,k}\\
     & = \frac{1}{K}\sum_{k=1}^K [U_k(P_{\text{fed}}^{\calRO}) - U_k(P_{\text{fed}}^{\calD})]\\  
    & = \frac{1}{K}\sum_{k=1}^K\left[\mathbb E_{w\sim P_{\text{fed}}^{\calRO}}[U_k(w)] - \mathbb E_{w\sim P_{\text{fed}}^{\calD}}[U_k(w)]\right]\\
    & = 0.
\end{align*}
\end{proof}

%For example, the trade-off between the privacy leakage and efficiency reduction is  

%\subsection{Homomorphic Encryption + Secret Sharing}

The communication cost for the model information of secret sharing mechanism is not guaranteed to satisfy \pref{assump: assump_of_Xi_efficiency}, and the analysis of the lower bound of the efficiency reduction is beyond the scope of our article.

%% file: App_C_4_Compression.tex
\section{Analysis for Compression Mechanism}\label{sec: analysis_for_Compression}
%\subsubsection{Privacy-Utility-Efficiency Trade-off}
To facilitate the analysis, we simplify the compression mechanism as follows. Let $b_i$ be a random variable sampled from Bernoulli distribution. The probability that $b_i$ is equal to $1$ is $\rho_i$.
\begin{equation}
W_k^{\calD}(i) =\left\{
\begin{array}{cl}
 W_k^{\calRO}(i) &  \text{if } b_i = 1,\\
0,  &  \text{if } b_i = 0.\\
\end{array} \right.
\end{equation}

Let $m$ represent the dimension of the model information.
\begin{itemize}
    \item Let $W_k^{\calRO}$ represent the original model information that follows a uniform distribution over $[a_k^1 - \delta, a_k^1 + \delta]\times [a_k^2 - \delta, a_k^2 + \delta]\cdots\times [a_k^{m} - \delta, a_k^{m} + \delta]$.
    \item Each dimension $i$ of the distorted model information $W_k^{\calD}(i)$ takes the value identical with that of $W_k^{\calRO}$ with probability $\rho_i$, and $0$ with probability $1 - \rho_i$.
    \item Let $W_{\text{fed}}^{\calRO}$ represent the federated plaintext that follows a uniform distribution over $[\widetilde a^1 - \delta, \widetilde a^1 + \delta]\times [\widetilde a^2 - \delta, \widetilde a^2 + \delta]\cdots\times [\widetilde a^{m} - \delta, \widetilde a^{m} + \delta]$, where $\widetilde a_i = \frac{1}{K}\sum_{k = 1}^K a_k^i$.
    \item Each dimension $i$ of the distorted model information $W_{\text{fed}}^{\calD}(i)$ takes the value identical with that of $W_{\text{fed}}^{\calRO}$ with probability $\rho_i$, and $0$ with probability $1 - \rho_i$.
\end{itemize}

\begin{lem}
For compression mechanism, the privacy leakage is lower bounded by
\begin{align*}
    \epsilon_p\ge C_1 - C_2\cdot \left(1 - \prod_{i = 1}^{m} \rho_i\right).
\end{align*}
The efficiency reduction is lower bounded by
\begin{align*}
\epsilon_e\ge \Xi\cdot\Gamma\cdot \left(1 - \prod_{i = 1}^{m}\rho_i\right).
\end{align*}
\end{lem}
\begin{proof}

The original model information $W_k^{\calRO}$ follows a uniform distribution over $[a_k^1 - \delta, a_k^1 + \delta]\times [a_k^2 - \delta, a_k^2 + \delta]\cdots\times [a_k^{m} - \delta, a_k^{m} + \delta]$, where $m$ represents the dimension of $W_k^{\calRO}$. Besides,

\begin{equation}
W_k^{\calD}(i) =\left\{
\begin{array}{cl}
 W_k^{\calRO}(i) &  \text{with probability } \rho_i,\\
0,  &  \text{with probability } 1- \rho_i.\\
\end{array} \right.
\end{equation}
Then we have that
\begin{align*}
    \text{TV}(P^{\calRO}_k || P^{\calD}_k) 
    &= \int_{[a_k^1 - \delta, a_k^1 + \delta]}\int_{[a_k^2 - \delta, a_k^2 + \delta]}\cdots\int_{[a_k^{m} - \delta, a_k^{m} + \delta]} \left(\left(\frac{1}{2\delta}\right)^{m} - \prod_{i = 1}^{m}\left(\frac{\rho_i}{2\delta}\right)\right) dw_1 dw_2 \cdots d{w_{m}}\\
    & = \left(\left(\frac{1}{2\delta}\right)^{m} - \prod_{i = 1}^{m}\left(\frac{\rho_i}{2\delta}\right)\right)\cdot (2\delta)^{m}.
\end{align*}
From \pref{lem: total_variation-privacy trade-off} we have that

\begin{align*}
\epsilon_{p}\ge\frac{1}{K}\sum_{k=1}^K \sqrt{{\text{JS}}(F^{\calRO}_k || F^{\calO}_k)} - \frac{1}{K}\sum_{k=1}^K \frac{1}{2}(e^{2\xi}-1)\cdot {\text{TV}}(P_k^{\calRO} || P^{\calD}_k).
\end{align*}
Combining the above two equations, we have that

\begin{align*}
    \epsilon_p &\ge \frac{1}{K}\sum_{k=1}^K \sqrt{{\text{JS}}(F^{\calRO}_k || F^{\calO}_k)} - \frac{1}{K}\sum_{k=1}^K \frac{1}{2}(e^{2\xi}-1)\cdot {\text{TV}}(P_k^{\calRO} || P^{\calD}_k)\nonumber\\ &=\sqrt{{\text{JS}}(F^{\calO}_k || F^{\calRO}_k)} - \frac{1}{K}\sum_{k=1}^K\frac{1}{2}(e^{2\xi}-1)\cdot \left(\left(\frac{1}{2\delta}\right)^{m} - \prod_{i = 1}^{m}\left(\frac{\rho_i}{2\delta}\right)\right)\cdot (2\delta)^{m}\nonumber\\
    & = C_1 - C_2\cdot \left(1 - \prod_{i = 1}^{m} \rho_i\right).
\end{align*}

% Assume that $W_k^{\calRO}$ follows a uniform distribution over $[a_k^1 - \delta, a_k^1 + \delta]\times [a_k^2 - \delta, a_k^2 + \delta]\cdots\times [a_k^{m} - \delta, a_k^{m} + \delta]$, where $m$ represents the dimension of $W_k^{\calRO}$. Besides,

% \begin{equation}
% W_k^{\calD}(i) =\left\{
% \begin{array}{cl}
%  W_k^{\calRO}(i) &  \text{if } b_i = 1,\\
% 0,  &  \text{if } b_i = 0.\\
% \end{array} \right.
% \end{equation}
From \pref{lem: efficiency_reduction_and_tvd}, we have that
\begin{align*}
    \epsilon_{e} &\ge\frac{1}{K}\sum_{k=1}^K\Xi\cdot\Gamma\cdot {\text{TV}}(P^{\calRO}_{k} || P^{\calD}_{k} )\\
    &\ge\Xi\cdot\Gamma\cdot\left(\left(\frac{1}{2\delta}\right)^{m} - \prod_{i = 1}^{m}\left(\frac{\rho_i}{2\delta}\right)\right)\cdot (2\delta)^{m}\\
    & = \Xi\cdot\Gamma\cdot \left(1 - \prod_{i = 1}^{m}\rho_i\right).
\end{align*}

\end{proof}

\begin{lem}
For compression mechanism, the utility loss is bounded by 
\begin{align*}
    \epsilon_u \ge \frac{\Delta}{2}\cdot \left(1 - \prod_{i = 1}^{m}\rho_i\right).
\end{align*}
\end{lem}
\begin{proof}
Recall that $W_{\text{fed}}^{\calRO}$ follows a uniform distribution over $[\widetilde a^1 - \delta, \widetilde a^1 + \delta]\times [\widetilde a^2 - \delta, \widetilde a^2 + \delta]\cdots\times [\widetilde a^{m} - \delta, \widetilde a^{m} + \delta]$, where $\widetilde a^i = \frac{1}{K} \sum_{k = 1}^K a_k^i$. Besides,

\begin{equation}
W_{\text{fed}}^{\calD}(i) =\left\{
\begin{array}{cl}
 W_{\text{fed}}^{\calRO}(i) &  \text{with probability } \rho_i,\\
0,  &  \text{with probability } 1- \rho_i.\\
\end{array} \right.
\end{equation}
Then we have that
\begin{align*}
    \text{TV}(P^{\calRO}_{\text{fed}} || P^{\calD}_{\text{fed}}) 
    &= \int_{[\widetilde a^1 - \delta, \widetilde a^1 + \delta]}\int_{[\widetilde a^2 - \delta, \widetilde a^2 + \delta]}\cdots\int_{[\widetilde a^{m} - \delta, \widetilde a^{m} + \delta]} \left(\left(\frac{1}{2\delta}\right)^{m} - \prod_{i = 1}^{m}\left(\frac{\rho_i}{2\delta}\right)\right) dw_1 dw_2 \cdots d{w_{m}}\\
    & = \left(\left(\frac{1}{2\delta}\right)^{m} - \prod_{i = 1}^{m}\left(\frac{\rho_i}{2\delta}\right)\right)\cdot (2\delta)^{m}.
\end{align*}
From \pref{lem: total_variation-utility trade-off_mt}, we have
\begin{align}\label{eq: converge_eq_2_0_app_02}
    \epsilon_{u} \ge \frac{\Delta}{2}\cdot {\text{TV}}(P^{\calRO}_{\text{fed}} || P^{\calD}_{\text{fed}} ).
\end{align}
Therefore, 
\begin{align*}
    \epsilon_u & \ge \frac{\Delta}{2}\cdot \left(\left(\frac{1}{2\delta}\right)^{m} - \prod_{i = 1}^{m}\left(\frac{\rho_i}{2\delta}\right)\right)\cdot (2\delta)^{m}\\
    & = \frac{\Delta}{2}\cdot \left(1 - \prod_{i = 1}^{m}\rho_i\right).
\end{align*}
\end{proof}